\documentclass[11pt]{article}

\usepackage[a4paper,margin=1in,headheight=14pt]{geometry}

\usepackage{amsmath}
\usepackage{amsthm}
\usepackage{amssymb,amsfonts}

\usepackage[T1]{fontenc}
\usepackage[utf8]{inputenc}
\usepackage{newpxtext}
\usepackage[amssymbols]{newpxmath}
\usepackage[scaled=0.92]{helvet}   
\usepackage{microtype}
\usepackage{mathtools}
\usepackage{bm}

\usepackage{graphicx}
\usepackage{multirow}
\usepackage{booktabs}
\usepackage{array}
\usepackage{tabularx}
\usepackage{subcaption}
\usepackage{enumitem}
\usepackage{algorithm}
\usepackage{algorithmicx}
\usepackage{algpseudocode}

\usepackage{xcolor}
\definecolor{outspan}{HTML}{E8744F}
\definecolor{inspan}{HTML}{159A8C}
\definecolor{accent}{HTML}{1F6F78}   
\definecolor{linkcol}{HTML}{1F6F78}
\definecolor{citecol}{HTML}{9B4F2E}

\usepackage{tcolorbox}
\tcbuselibrary{skins,breakable}

\usepackage{titlesec}
\titleformat{\section}
  {\normalfont\sffamily\large\bfseries\color{accent}}
  {\thesection}{0.75em}{}
\titleformat{\subsection}
  {\normalfont\sffamily\normalsize\bfseries\color{accent!85!black}}
  {\thesubsection}{0.6em}{}
\titlespacing*{\section}{0pt}{1.6ex plus 0.6ex minus .2ex}{0.9ex plus .2ex}
\titlespacing*{\subsection}{0pt}{1.3ex plus 0.5ex minus .2ex}{0.6ex plus .2ex}

\usepackage[font=small,labelfont={bf,sf},labelsep=period,
            justification=justified,skip=6pt]{caption}
\usepackage{fancyhdr}
\fancypagestyle{titlefooter}{%
  \fancyhf{}%
  \fancyfoot[L]{\footnotesize\sffamily\color{gray}Hedayat et al.}%
  \fancyfoot[C]{\thepage}%
}

\usepackage[colorlinks=true,
            linkcolor=linkcol,
            citecolor=citecol,
            urlcolor=linkcol,
            breaklinks=true]{hyperref}
\usepackage[capitalise,nameinlink]{cleveref}
\crefname{figure}{Fig.}{Figs.}
\Crefname{figure}{Fig.}{Figs.}
\crefname{table}{Table}{Tables}
\Crefname{table}{Table}{Tables}
\crefname{equation}{Eq.}{Eqs.}
\Crefname{equation}{Eq.}{Eqs.}
\crefname{section}{Section}{Sections}
\Crefname{section}{Section}{Sections}
\crefname{appendix}{Appendix}{Appendices}
\Crefname{appendix}{Appendix}{Appendices}
\crefname{algorithm}{Algorithm}{Algorithms}
\Crefname{algorithm}{Algorithm}{Algorithms}
\crefname{theorem}{Theorem}{Theorems}
\Crefname{theorem}{Theorem}{Theorems}
\crefname{definition}{Definition}{Definitions}
\Crefname{definition}{Definition}{Definitions}

\theoremstyle{definition}
\newtheorem{statement}{Statement}
\newtheorem{theorem}[statement]{Theorem}

\newtheorem{definition}[statement]{Definition}
\theoremstyle{remark}

\graphicspath{{./}}

\allowdisplaybreaks
\begin{document}
\thispagestyle{titlefooter}

\begin{center}
  {\color{accent}\rule{\linewidth}{1.2pt}}\\[1.1em]
  {\sffamily\bfseries\LARGE
   In-context learning from self-generated\\[0.25em]
   trajectories for adaptive model reduction\par}
  \vspace{0.9em}
  {\color{accent}\rule{\linewidth}{0.6pt}}\\[1.4em]

  \renewcommand{\thefootnote}{\fnsymbol{footnote}}
  {\large
   Amirpasha Hedayat\textsuperscript{1,3,}\footnote{Corresponding author:
   \href{mailto:ahedayat@umich.edu}{ahedayat@umich.edu}}\quad
   Laura Balzano\textsuperscript{2,3}\quad
   Karthik Duraisamy\textsuperscript{1,3}\par}
  \renewcommand{\thefootnote}{\arabic{footnote}}\setcounter{footnote}{0}
  \vspace{0.9em}

  {\small\itshape
   \textsuperscript{1}Department of Aerospace Engineering\\
   \textsuperscript{2}Department of Electrical Engineering and Computer Science\\
   \textsuperscript{3}Michigan Institute for Computational Discovery and Engineering,
   University of Michigan, Ann Arbor, MI, USA\par}
  \vspace{1.6em}
\end{center}

\begingroup
\setlength{\leftskip}{0.06\linewidth}
\setlength{\rightskip}{0.06\linewidth}
\noindent{\sffamily\bfseries\color{accent}Abstract}\\[0.4em]
\noindent
High-fidelity simulations of complex physical systems are often too expensive for repeated prediction, design, and control. Reduced-order models address this computational cost by compressing high-dimensional dynamics into low-dimensional representations that can be evaluated rapidly, but they often lose accuracy when online dynamics drift beyond the offline training data. Adaptive methods address this limitation by updating the reduced subspace online using external, \emph{out-of-span} information, such as full-order corrections or sensor snapshots. We discovered that a complementary and previously unexploited \emph{in-span} adaptation channel exists within the current reduced subspace. To exploit this channel, we continually update the reduced representation using the model's own predictions through an incremental singular-value decomposition with a forgetting factor. This produces a \emph{trajectory-informed spectral preconditioner} in which the reduced subspace remains unchanged, while the basis is reweighted and realigned according to the directions visited by the evolving dynamics. This internal reorganization prepares the basis to absorb future out-of-span corrections more effectively, improving adaptation without requiring additional external information. We expose the mechanism through a three-dimensional spiral example and demonstrate its benefits on nonlinear partial differential equations, including viscous Burgers and Fisher--KPP dynamics. We also discuss how in-span learning can be interpreted as a dynamical-systems analogue of in-context learning. More broadly, in-span learning suggests a new principle for computational science, revealing that model-generated trajectories contain more usable information than previously recognized.
\endgroup

\vspace{1.0em}
\noindent{\color{accent}\rule{\linewidth}{0.6pt}}
\vspace{0.4em}

\begingroup
\footnotesize\color{gray}\itshape
\noindent This is a preprint.
\endgroup

\vspace{1.2em}

\section{Introduction}

Across science and engineering, applications such as aircraft design~\cite{lyu2015aerodynamic}, weather forecasting~\cite{lam2023learning,kochkov2024neural,bodnar2025foundation}, combustion predictions~\cite{swischuk2020learning}, power-grid control~\cite{khaloie2025review}, uncertainty quantification~\cite{xiu2003modeling,degen2022multiphysics}, and digital twins~\cite{niederer2021scaling,kapteyn2021probabilistic} all rely on numerical models that  require expensive computations. There is also a need to  run these computations repeatedly across parameters, controls, uncertain inputs, or real-time decisions. This setting calls for simulators that are accurate enough to trust and inexpensive enough to use freely.

Reduced-order models (ROMs) address this need by compressing high-dimensional simulations into low-dimensional representations~\cite{antoulas2005approximation,rowley2017model,antoulas2020interpolatory,kramer2024learning}. Although a discretized physical model may contain millions of degrees of freedom, its relevant trajectories often occupy a much smaller region of state space. Projection-based ROMs exploit this structure by constructing a low-dimensional subspace, commonly from proper orthogonal decomposition (POD) of full-order snapshots~\cite{lumley1967structure,sirovich1987turbulence,holmes1996turbulence}, and evolving only the coordinates of the solution in that subspace~\cite{benner2015survey}. When the online dynamics remain close to the training data, this compression can deliver orders-of-magnitude speed-ups~\cite{huang2022component}.
The same compression also creates the central failure mode of ROMs. A static ROM is trained offline and deployed with a fixed basis. It can interpolate accurately within the regime represented by its snapshots, but it loses predictive power when the online trajectory drifts beyond that regime. This is a form of distribution shift that exists whether the reduced representation is a linear subspace, a nonlinear manifold~\cite{geelen2023operator}, or a learned surrogate~\cite{lee2020model}. For projection-based ROMs considered in this work, this limitation is especially severe for systems with slowly decaying Kolmogorov $N$-width, where no fixed low-dimensional linear subspace can represent the evolving solution uniformly well over long times~\cite{peherstorfer2022breaking,barnett2022quadratic}. Transport-dominated flows, shocks, moving fronts, combustion, and detonation dynamics are typical examples~\cite{reiss2018shifted,nair2019transported,rim2023manifold,mirhoseini2023model,arnoldmedabalimi2022large}. They are also among the settings where fast predictive simulation would be most valuable.

Adaptive ROMs attempt to overcome this rigidity by allowing the reduced representation to evolve during the online simulation. Early efforts included interpolation between locally trained bases~\cite{amsallem2008interpolation,amsallem2011online,amsallem2012nonlinear,peherstorfer2014localized}. Pioneered by Peherstorfer \& Willcox~\cite{peherstorfer2015dynamic,peherstorfer2015online}, incremental basis approaches have gained attention, where the basis is updated online with newly acquired information~\cite{peherstorfer2020model,huang2023predictive,singh2023lookahead,mohaghegh2026feature,hedayat2026adaptive_nonintrusive}. A related line of work draws on subspace-tracking ideas~\cite{balzano2018streaming,bunch1978updating,brand2002incremental,brand2006fast,oja1982simplified,balzano2010online,yang1995projection,chi2013petrels}, which have motivated geometric basis updates~\cite{zimmermann2018geometric} and history-aware adaptation~\cite{hedayat2026history}.
These incremental basis-updating approaches differ in how they obtain new information and how they update the basis, but they share a common premise. The basis is adapted using information with a nonzero component outside the current reduced subspace. This external information may come from a full-order state, a sensor snapshot, or another correction oracle. We call such data \emph{out-of-span}. The premise is natural because only out-of-span information can move the subspace. A snapshot that lies exactly inside the current subspace has zero residual and therefore appears, from the usual subspace-update viewpoint, to add no value.

In this paper, we show that this reasoning misses a complementary adaptation channel. Between external corrections, a ROM continuously produces predictions of its own state. By construction, these predictions lie inside the current reduced subspace, and we therefore call them \emph{in-span}. They cannot expand the subspace, but they can still reorganize the representation inside it. We show that when the ROM's own predictions are streamed through an incremental singular-value decomposition (iSVD)~\cite{bunch1978updating,brand2002incremental,brand2006fast} with forgetting, they rotate the orthonormal basis within the same subspace and reweight its singular spectrum. The result is a \emph{trajectory-informed spectral preconditioner}. It does not change what the ROM can represent at that instant, but it changes how the basis is prepared to absorb the next external correction. We refer to the resulting adaptive ROM as \emph{SPIN}, for \emph{\textbf{S}pectral \textbf{P}reconditioning via \textbf{IN}-span learning} (\cref{fig:concept}). Throughout the paper, we use \emph{baseline adaptive ROM} for the iSVD-based adaptive model that only uses external out-of-span corrections and performs no in-span updates between correction events. It is notable that the baseline adaptive ROM is itself representative of the state-of-the-art~\cite{hedayat2026history}.

\begin{figure}[!t]
\centering
\includegraphics[width=\linewidth]{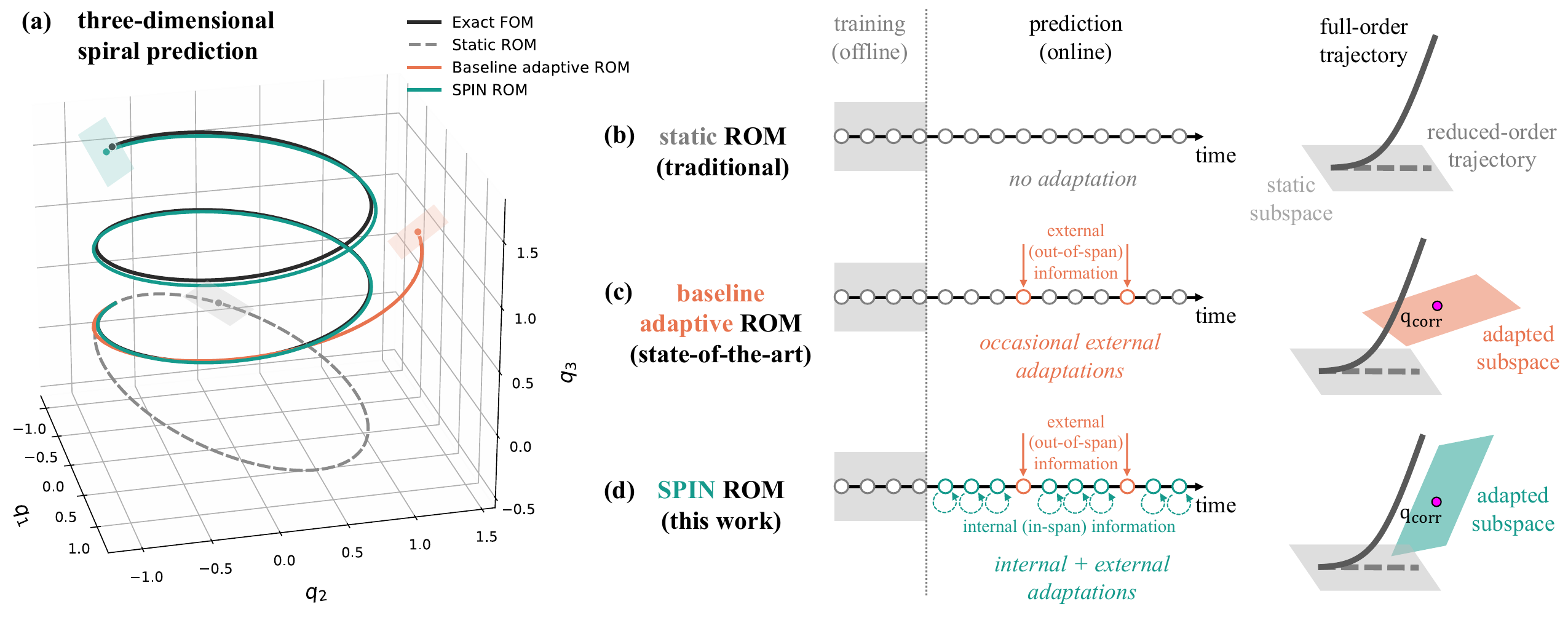}
\caption{\textbf{In-span learning adapts a reduced-order model from its own predictions.} Three reduced-order modeling strategies are shown through their offline/online timelines and the corresponding evolution of the reduced representation. \textbf{(a)} Three-dimensional spiral prediction comparing the exact full-order model (FOM) trajectory, static ROM, baseline adaptive ROM, and SPIN ROM. \textbf{(b)} A static ROM is trained offline and then deployed online with no adaptation. Its reduced trajectory remains confined to the original static subspace and can lose track of the full-order trajectory as the dynamics leave the training regime. \textbf{(c)} A baseline adaptive ROM queries external out-of-span information during the online phase, such as a correction snapshot $q_{\mathrm{corr}}$, and uses it to move the subspace toward the full-order dynamics. \textbf{(d)} The proposed SPIN ROM uses the same external correction events, but also streams its own intermediate predictions as in-span information. These in-span updates do not move the subspace by themselves, but they reweight and reorient the basis before the next correction arrives, preparing the reduced representation to absorb out-of-span information more effectively.}
\label{fig:concept}
\end{figure}

We develop in-span learning first at the level of the update mechanism, then in a closed-form spiral example, and finally in nonlinear PDEs where the same mechanism improves long time-horizon prediction. Across these examples, the same conclusion emerges. Reduced models can learn not only from new data outside their span, but also from the trajectory they have already produced inside it.

\section{Results}

\subsection{A hidden adaptation channel inside the current subspace}

We start with a brief overview of the in-span learning mechanism, showing how a model's own predictions can contribute to its adaptation. A summary of the notation used throughout the paper is provided in \cref{app:notation}. Consider a ROM with orthonormal basis $\Phi\in\mathbb{R}^{N\times r}$ and singular values collected in $\Sigma\coloneqq\mathrm{diag}(\sigma_1,\ldots,\sigma_r)$. At any time step, the ROM prediction has the form
$
    q_{\mathrm{ROM}}\coloneqq\Phi a,
$
where $a\in\mathbb{R}^r$ is the reduced coordinate vector. This prediction lies exactly in the current reduced subspace, so an iSVD update sees zero orthogonal residual and cannot add a new direction. The point of in-span learning is that this zero-residual snapshot can still change the basis state. Let $C\coloneqq\Sigma^2$ denote the current covariance represented in reduced coordinates. With forgetting factor $\gamma_{\rm in}\in(0,1]$, the in-span iSVD update reduces to a rank-one update of this reduced covariance,
\begin{equation}
    C^+ = \gamma_{\rm in}^2 C + aa^\top .
    \label{eq:inspan_covariance_update_results}
\end{equation}
Diagonalizing this updated small $r\times r$ matrix,
$C^+=U\Lambda U^\top$, gives the new squared singular values through $\Lambda$ and an orthogonal rotation $U\in\mathbb{R}^{r\times r}$ of the basis inside the same subspace. Equivalently,
\begin{equation}
    \Sigma^+ = \Lambda^{1/2},
    \qquad
    \Phi^+ = \Phi U,
    \qquad
    \mathrm{span}(\Phi^+) = \mathrm{span}(\Phi).
    \label{eq:inspan_subspace_preservation_results}
\end{equation}
Thus, an in-span update preserves the geometric object that determines what states the ROM can represent, but changes the spectral object that determines the modal weights. The formal iSVD update, the derivation of Eq.~\eqref{eq:inspan_covariance_update_results}, and the proof of Eq.~\eqref{eq:inspan_subspace_preservation_results} are given in \cref{sec:methods_isvd} and~\ref{sec:methods_isl}.

Equation~\eqref{eq:inspan_covariance_update_results} gives the mechanism a simple interpretation. The vector $a$ records which directions of the current basis the ROM trajectory is using. The update discounts the previous spectral state by $\gamma_{\rm in}^2$ and adds the rank-one contribution $aa^\top$ from the current ROM prediction. A direction that is strongly represented in the recent trajectory can be reinforced, whereas a direction that receives too little trajectory energy can decay. The outcome is trajectory-dependent and controlled by the forgetting factor, which sets the memory horizon of in-span learning. With $\gamma_{\rm in}$ close to one, the ROM accumulates many recent predictions and reorganizes the basis around sustained dynamics. With smaller $\gamma_{\rm in}$, it emphasizes the most recent trajectory and can rapidly suppress inactive directions. The reinforcement--suppression threshold is derived in \cref{sec:methods_isl}.

This in-span update matters because the adaptive ROMs studied here operate at fixed rank. When the next external correction arrives, the new out-of-span information must be merged with the existing basis and then truncated back to rank $r$. The outcome of that truncation depends not only on the incoming correction, but also on the spectral state and orientation of the basis that receives it. In-span learning therefore acts as a trajectory-informed spectral preconditioner for the next out-of-span update. This mechanism uses no additional full-order information. Between external corrections, the ROM already produces the states used in the in-span update. The only extra operation is a small SVD in the reduced space, followed by a rotation and reweighting of the existing basis and the usual refresh of basis-dependent quantities in the hyper-reduced model when needed. The cost and implementation details are described in \cref{sec:methods_complexity} and~\ref{sec:methods_alg}.

\subsection{A toy example to expose a key mechanism}

To illustrate in-span learning in a simple setting, we begin with a three-dimensional spiral whose full-order dynamics are known exactly as
$
    q(t)\coloneqq
    \begin{bmatrix}
    \cos t \quad \sin t \quad \alpha t
    \end{bmatrix}^T
$ with
$
    \alpha=0.4.
$
This trajectory solves an affine linear system in $\mathbb{R}^3$, which we advance with its exact discrete flow. The full setup is given in \cref{sec:methods_fom}.

The initial ROM is deliberately minimal. We form a rank-$2$ POD basis $\Phi_0$ from only two full-order snapshots, $q^1$ and $q^2$, and project the exact discrete dynamics onto the resulting plane $\mathcal{S}_0:=\mathrm{span}(\Phi_0)=\mathrm{span}(q^1,q^2)$. The static ROM trajectory remains confined to $\mathcal{S}_0$ and drifts as the true spiral leaves this plane. This small example already contains the essential failure mode of a static ROM, in which the model can evolve only inside the space it was given.

The in-span updates use only the ROM trajectory. At an intermediate online time step $k$, before the next external correction, the ROM produces
$
    q^k_{\mathrm{ROM}}\coloneqq\Phi_0 a^k\in\mathcal{S}_0,
$
where $a^k\in\mathbb{R}^2$ is the reduced coordinate vector in the initial basis $\Phi_0$. Let $n_1,\ldots,n_s$ denote the time indices of the $s$ in-span ROM predictions streamed since the previous external correction. After these $s$ updates, the reduced covariance is
\begin{equation}
    C_s=\gamma_{\rm in}^{2s}\Sigma_0^2+
    \sum_{j=1}^{s}\gamma_{\rm in}^{2(s-j)}a^{n_j}(a^{n_j})^\top ,
    \label{eq:spiral_covariance_results}
\end{equation}
where $\Sigma_0$ contains the initial singular values associated with $\Phi_0$. We use a small value $\gamma_{\rm in}=0.1$ to make the spectral contraction visible in the spiral illustration; in the PDE experiments, the in-span forgetting factors are chosen $\approx 1$. Writing $C_s=U_s\Lambda_s U_s^\top$ gives the preconditioned singular values through $\Lambda_s^{1/2}$ and the in-plane basis rotation through
$
    \Phi_s=\Phi_0 U_s,
$ with
$
    \mathrm{span}(\Phi_s)=\mathcal{S}_0.
$

\Cref{fig:spiral_setup} shows this effect geometrically. The full-order spiral continues in three dimensions, while the ROM trajectory evolves inside the initial rank-2 plane. In this particular example, the reduced trajectory repeatedly visits one dominant in-plane direction while leaving the other nearly inactive. The covariance ellipse therefore rotates toward the visited direction and contracts along the direction that receives little trajectory energy. The initial singular values
$
    \Sigma_0=\mathrm{diag}(2.92,\;0.25)
$
are transformed by four in-span updates into
$
    \Sigma_{\mathrm{SPIN}}=
    \mathrm{diag}(2.48,\;0.03).
$
The second singular value is reduced by almost an order of magnitude, while the subspace itself is unchanged.

\begin{figure}[!t]
\centering
\includegraphics[width=\linewidth]{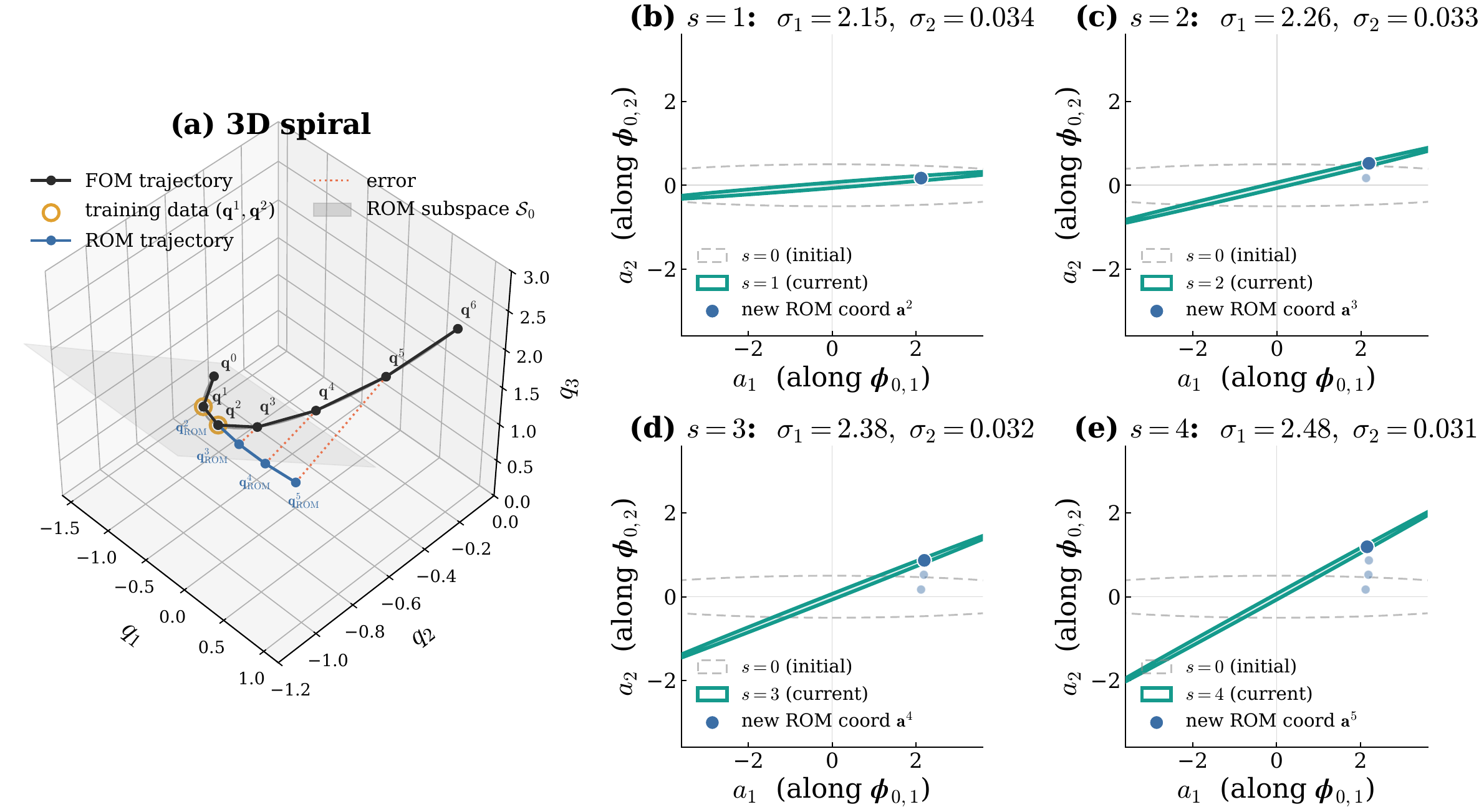}
\caption{\textbf{A closed-form spiral reveals in-span spectral preconditioning.}
\textbf{(a)} The exact full-order spiral trajectory (FOM), the rank-2 ROM trajectory advanced inside its fixed initial subspace $\mathcal{S}_0$, and the growing discrepancy (error) between them; the two snapshots $q^1,q^2$ used to build the initial basis are highlighted, and the grey plane marks $\mathcal{S}_0$.
\textbf{(b)--(e)} Covariance ellipses inside $\mathcal{S}_0$ after $s=1,2,3,4$ in-span updates, drawn in the fixed $\mathcal{S}_0$ coordinate frame with axes along $\boldsymbol{\phi}_{0,1}$ and $\boldsymbol{\phi}_{0,2}$. Each ellipse is set by the eigenvectors and eigenvalues of the reduced covariance $C_s$. Although the in-span updates use only ROM-predicted states lying in $\mathcal{S}_0$ and therefore add no new direction, the spectrum rotates and contracts. The active direction is reinforced while the unvisited direction is progressively suppressed. The subspace is unchanged; only the basis orientation and singular spectrum evolve.}
\label{fig:spiral_setup}
\end{figure}

Here, we separate two objects that are often conflated. The \emph{subspace} determines which states the ROM can represent, while the \emph{spectrum} assigns importance to the directions already present within that subspace. The in-span recursion leaves the first object unchanged but modifies the second. On its own, this cannot correct the ROM, because no new direction has entered the span. Its effect becomes visible when an out-of-span correction later arrives.

\subsection{In-span preconditioning changes how the next correction is absorbed}
To receive an out-of-span correction without having access to the exact full-order trajectory, we perform a full-order operator query. In the spiral example, the exact one-step flow is the affine map $q^{k+1}=Gq^k+c$, where $G$ and $c$ are obtained from the exact spiral dynamics in \cref{sec:methods_fom}. Starting from the ROM state $q^5_{\mathrm{ROM}}$, this gives
$
    q_{\mathrm{corr}} \coloneqq G q^5_{\mathrm{ROM}} + c .
$
At the first out-of-span correction, this snapshot is identical for the baseline adaptive ROM and the SPIN ROM. Its residual relative to the initial ROM plane has norm
$
    \rho \coloneqq \|q_{\mathrm{corr}}-\Phi_0\Phi_0^\top q_{\mathrm{corr}}\|_2
    =0.42 ,
$
so both methods receive the same out-of-span information. The difference is the basis state that receives it. In the baseline adaptive ROM, the correction is applied to the original spectrum; in the SPIN ROM, it is applied after the ROM trajectory has preconditioned the spectrum.

We now show that, in this example, in-span preconditioning acts as a weak-mode suppressor. The ROM trajectory has reduced the strength of a nearly inactive in-plane direction, so the incoming residual faces a smaller modal competitor when the update is truncated back to rank $r$. To quantify this local rank competition, we compare the residual strength $\rho^2$ with the weakest retained in-plane direction before the residual is added. For a correction with in-plane coefficients $\alpha=\Phi_{\mathrm{pre}}^\top q_{\mathrm{corr}}$, define
$
    C_{\rm corr} \coloneqq \gamma_{\rm out}^2 \Sigma_{\mathrm{pre}}^2 + \alpha\alpha^\top ,
$
which is the in-plane block of the small iSVD core covariance. We then define the residual-to-weakest-mode ratio
$
    \eta \coloneqq \frac{\rho^2}{\lambda_r(C_{\rm corr})} .
$
Here, $\eta\ll 1$ indicates that the residual is weak relative to the existing in-plane core content, while $\eta\gg 1$ indicates that it can compete for one of the retained rank-$r$ slots. This ratio is used only as a diagnostic of the local rank competition; the actual truncated update is determined by the full core SVD, which is examined in \cref{app:spiral_details}.

For the same correction snapshot, and with $\gamma_{\rm out}=1.0$ for simplicity, the baseline adaptive update gives
$
    \lambda_r(C_{\mathrm{corr,\,baseline}})=1.45
$
and
$
    \eta_{\mathrm{baseline}}\approx 0.12,
$
whereas the SPIN update gives
$
    \lambda_r(C_{\mathrm{corr,\,SPIN}})=0.05
$
and
$
    \eta_{\mathrm{SPIN}}\approx 3.67.
$
The residual norm has not changed. The increase in $\eta$ comes from the denominator, meaning that the same residual now competes against a much smaller in-plane core direction.

The geometric outcome is shown in \cref{fig:rank_competition}. With out-of-span adaptation alone, the updated plane changes only modestly and the new residual is mostly discarded. With in-span preconditioning, the updated plane swings toward the new local direction of the spiral. Quantitatively, enabling in-span learning increases the residual capture from $0.26$ to $0.80$ and the plane-change angle from $15.06^\circ$ to $53.29^\circ$, while decreasing the correction error from $1.02\times 10^{-2}$ to $8.84\times 10^{-4}$. These diagnostics are defined in \cref{sec:methods_diagnostics}. Full numerical details for the spiral example are provided in \cref{app:spiral_details}, and a comparison with several alternative fixed-rank basis updates is available in \cref{app:alternative_updates}.

\begin{figure}[!t]
\centering
\includegraphics[width=\linewidth]{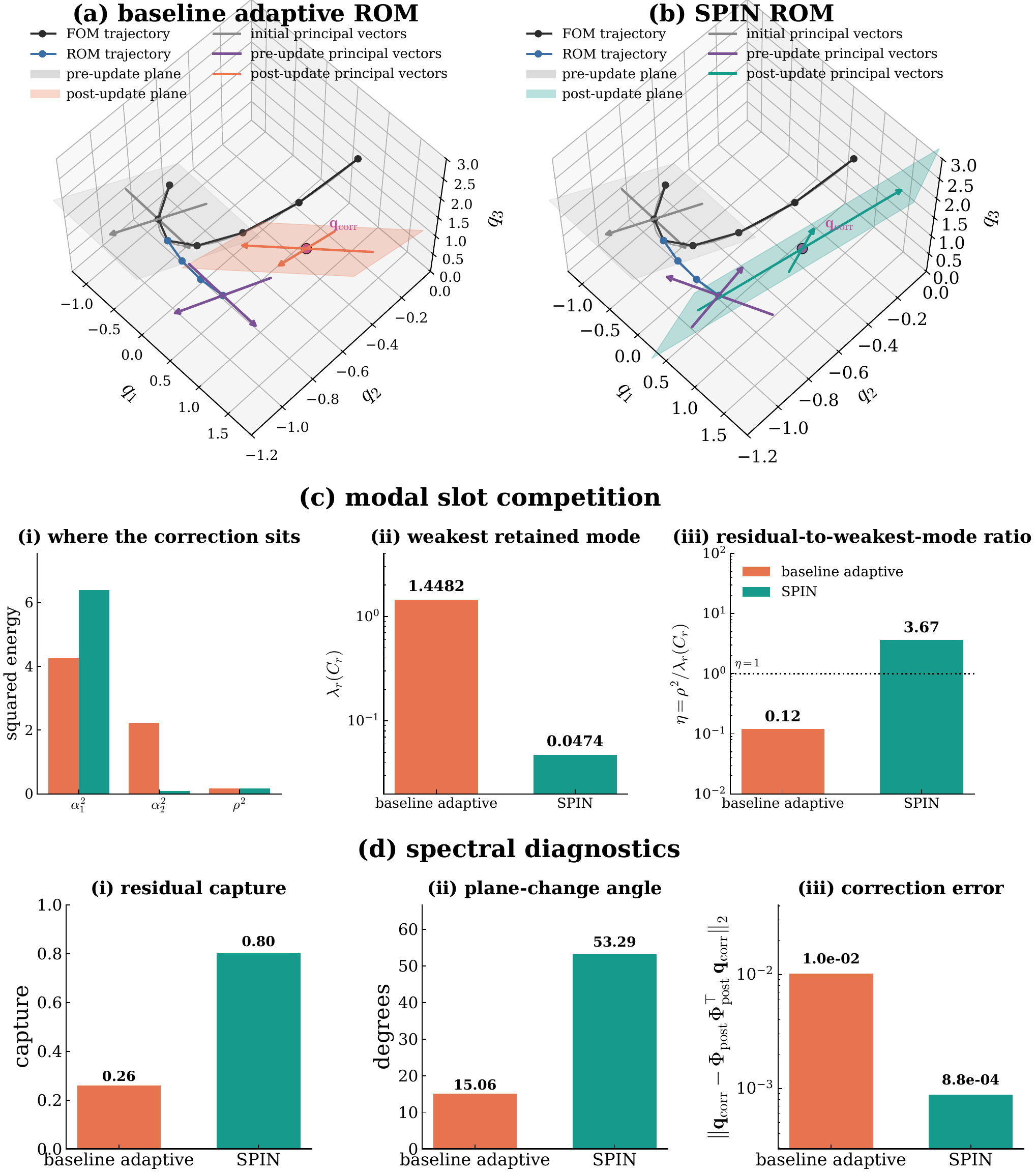}
\caption{
\textbf{The same correction is better absorbed after in-span preconditioning.}
Both adaptive ROMs receive the same correction $q_{\mathrm{corr}}$ and use the same truncated iSVD update; only the incoming basis state differs.
\textbf{(a,b)} Baseline adaptive and SPIN updates. In-span preconditioning rotates and reweights the basis before the correction arrives, so the post-update plane swings more strongly toward the local dynamics.
\textbf{(c)} Modal slot competition, including correction energy, $\lambda_r(C)$, and $\eta=\rho^2/\lambda_r(C)$.
\textbf{(d)} Residual capture, plane-change angle, and correction error.
In this spiral example, in-span preconditioning suppresses inactive modal content so the same residual can occupy the weak slot of the fixed-rank basis.
}
\label{fig:rank_competition}
\end{figure}

\subsection{SPIN improves nonlinear PDE prediction}

The spiral example isolates the mechanism, but it does not test whether the same effect survives the numerical ingredients of a practical ROM. We therefore consider two nonlinear PDEs, namely viscous Burgers and Fisher--KPP, solved with implicit time integration and hyper-reduced (via QDEIM~\cite{chaturantabut2010nonlinear, drmac2016new}) least-squares Petrov--Galerkin (LSPG)~\cite{carlberg2011efficient, carlberg2017galerkin} ROMs. This setting tests in-span learning within a broadly used framework for nonlinear scientific simulation. Viscous Burgers dynamics tests transport of a coherent structure, while Fisher--KPP reaction--diffusion tests front propagation and saturation driven by a nonlinear reaction term. In both cases, all reduced models are initialized from a negligible offline training phase, so the comparison is deliberately extrapolative.

For both PDEs we compare the full-order model, a static ROM, the baseline adaptive ROM, and the proposed SPIN ROM. The two adaptive ROMs use the same rank, QDEIM sample count, correction interval, and full-order operator query mechanism to receive the correction snapshot. The only algorithmic difference is that SPIN also updates its basis and singular values between external corrections using the ROM's own intermediate predictions. The full discretizations, LSPG--QDEIM formulation, correction procedure, and hyperparameters are given in \cref{sec:methods_rom} and~\ref{sec:methods_query}.

For each adaptive model, the forgetting factors reported here are selected from an ablation sweep that minimizes the mean relative $L_2$ error over the online trajectory. Thus, the comparison is between optimized memory configurations of the two adaptive strategies. For Burgers, the best baseline adaptive ROM uses $\gamma_{\rm out}=0.01$, while the best SPIN ROM uses $(\gamma_{\rm in},\gamma_{\rm out})=(1.0,0.25)$. For Fisher--KPP, the corresponding values are $\gamma_{\rm out}=0.1$ and $(\gamma_{\rm in},\gamma_{\rm out})=(0.9,0.25)$. The full forgetting-factor ablations, correction-interval sweeps, and additional PDE robustness tests are reported in \cref{app:hyperparameter_sensitivity}, with the benchmark parameters summarized in \cref{app:reproducibility} for reproduction purposes.

\Cref{fig:burgers_prediction} shows Burgers predictions, where we use a rank-$4$ ROM with $m=4$ QDEIM sample points, trained from only the first four full-order snapshots from $t=0$ to $t=0.004$, and evaluated over a $500$-step prediction horizon from $t=0$ to $t=0.50$, with out-of-span corrections every $z_s=10$ time steps. The static ROM fails quickly because its basis represents only the early offline regime. The baseline adaptive ROM performs substantially better and follows the advecting pulse reasonably well, showing that the external corrections supply useful missing directions. Adding the in-span channel further reduces amplitude and shape errors as the pulse moves across the domain, and the SPIN error remains consistently below the baseline adaptive error without adding full-order correction steps.

Fisher--KPP in \cref{fig:fisher_prediction} provides a more demanding test. A localized initial pulse grows through the reaction term, develops two traveling fronts, and eventually saturates over much of the domain. We again use a rank-$4$ LSPG--QDEIM ROM with $m=4$ sample points, trained from only the first four full-order snapshots from $t=0$ to $t=0.004$, and evaluated over a $500$-step prediction horizon from $t=0$ to $t=0.50$, with out-of-span corrections every $z_s=25$ time steps. The static ROM misses the expanding front after the solution leaves the training window. The baseline ROM improves on the static model, but develops visible oscillations in the saturated region and front-location errors. In contrast, the SPIN ROM follows both the plateau and the moving fronts more closely, with a larger separation between the adaptive error curves than in Burgers.

\begin{figure}[!t]
\centering
\includegraphics[width=\linewidth]{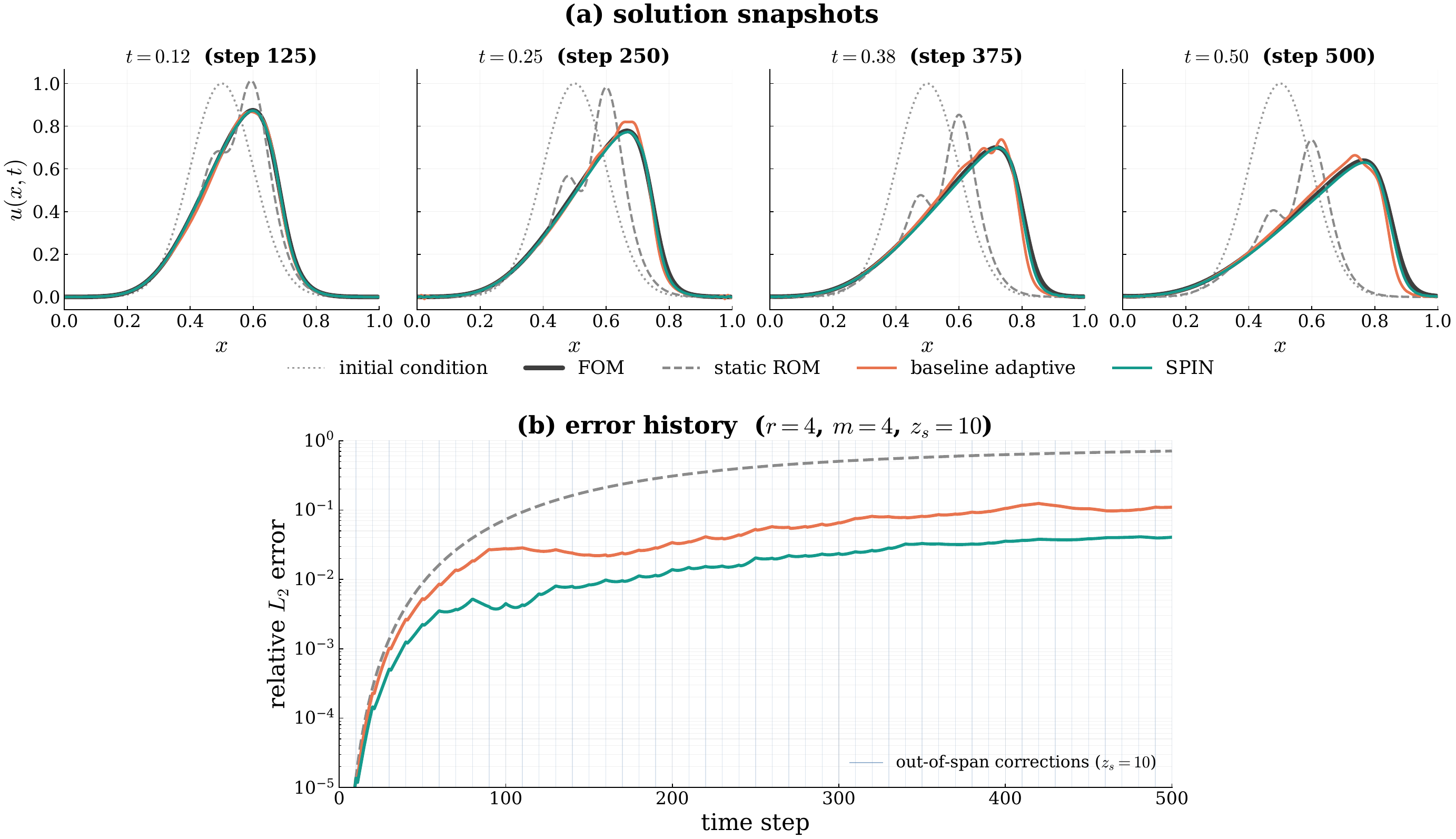}
\caption{\textbf{SPIN ROM improves viscous Burgers prediction.}
Viscous Burgers dynamics with rank $r=4$ and $m=4$ QDEIM sample points. The offline basis is trained on the initial $4$ time steps from $t=0$ to $t=0.004$, while prediction is evaluated over $500$ time steps from $t=0$ to $t=0.50$, with external corrections every $z_s=10$ steps.
\textbf{(a)} Solution profiles $u(x,t)$ at four reported times. The static ROM remains tied to the early training regime and drifts away from the transported pulse. The baseline adaptive ROM follows the FOM more closely, while the SPIN ROM further reduces the amplitude and shape errors.
\textbf{(b)} Relative $L_2$ error over time. Faint vertical marks indicate the out-of-span correction events. The SPIN ROM uses the same rank, sample count, correction interval, and number of full-order correction steps as the baseline adaptive ROM, but maintains a lower error throughout the trajectory.}
\label{fig:burgers_prediction}
\end{figure}

\begin{figure}[!t]
\centering
\includegraphics[width=\linewidth]{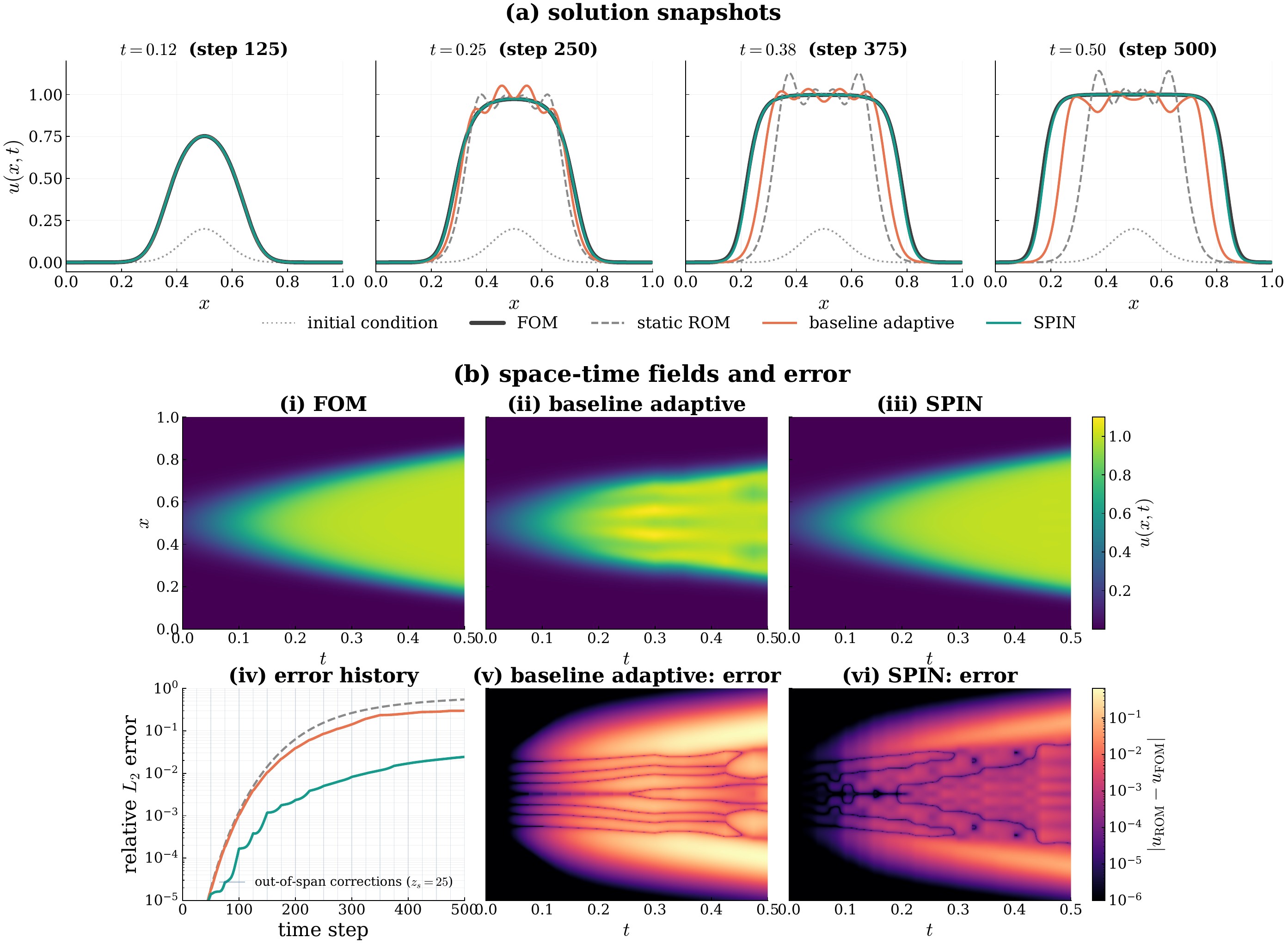}
\caption{\textbf{SPIN ROM improves Fisher--KPP front propagation.}
Fisher--KPP reaction--diffusion dynamics with rank $r=4$ and $m=4$ QDEIM sample points. The offline basis is trained on the initial $4$ time steps from $t=0$ to $t=0.004$, while prediction is evaluated over $500$ time steps from $t=0$ to $t=0.50$, with external corrections every $z_s=25$ steps.
\textbf{(a)} Solution profiles $u(x,t)$ at four reported times. The static ROM misses the expanding fronts after the solution leaves the training window. The baseline adaptive ROM improves on the static model but develops oscillations in the saturated region, while the SPIN ROM tracks the plateau and moving fronts more closely.
\textbf{(b)} Space--time fields and error. The FOM, baseline adaptive ROM, and SPIN ROM are shown as fields over $(x,t)$, together with the relative $L_2$ error history and pointwise absolute error fields. The SPIN ROM produces a visibly cleaner space--time structure and substantially smaller error over most of the trajectory.}
\label{fig:fisher_prediction}
\end{figure}

\subsection{Spectral diagnostics confirm the in-span mechanism}

The PDE results show that SPIN adaptation improves prediction. We next examine whether the same spectral mechanism identified in the spiral also appears in these nonlinear systems. We therefore measure three diagnostics that separate modal reweighting, in-plane basis rotation, and genuine subspace motion, computed along the online trajectory for the baseline adaptive and SPIN ROMs. These results are shown in \cref{fig:diagnostics} for both PDEs.

The first diagnostic is the time history of the singular values carried by the adaptive basis. In a baseline adaptive ROM, these values change only at external correction events. In a SPIN ROM, they are updated at every intermediate step by the ROM trajectory itself, directly reflecting the reduced covariance update in Eq.~\eqref{eq:inspan_covariance_update_results}. Unlike the first spiral correction, where one nearly inactive direction is strongly suppressed, the PDE spectra show richer trajectory-dependent behavior, where several retained modes are repeatedly reinforced between correction events. This is consistent with  $\gamma_{\rm in} \approx 1.0$ for Burgers and Fisher--KPP, which imply little or no in-span forgetting. The diagnostic message is therefore that successful in-span learning makes the spectral state trajectory-informed.

The second diagnostic measures basis reorientation through
$
    \left\| I - \Phi_{k+1}^{\top}\Phi_k \right\|_F .
$
This quantity is sensitive to changes in the orthonormal coordinate frame. It can be nonzero even when the two bases span essentially the same subspace. For the baseline adaptive ROM, the diagnostic is zero between correction events and spikes when an external correction updates the basis. For the SPIN ROM, it remains nonzero throughout the intervals between corrections. This confirms that the ROM predictions are not merely rescaling singular values, and they continuously rotate the basis inside the current subspace.

The third diagnostic isolates genuine subspace change. We compute the principal angles $\theta_i$ between consecutive subspaces from the singular values of $\Phi_k^\top\Phi_{k+1}$ and report
$
    \max_i \sin \theta_i .
$
This quantity is insensitive to a pure rotation of the coordinate frame inside the same subspace. Between external corrections, it remains close to numerical zero for the SPIN ROM, confirming that in-span updates preserve the subspace up to numerical precision. At correction events, it jumps for both adaptive models, because only then does out-of-span information enter. The jump is not identical for the two methods, showing that the same type of external event acts on different preconditioned basis states.

\begin{figure}[!t]
\centering
\includegraphics[width=\linewidth]{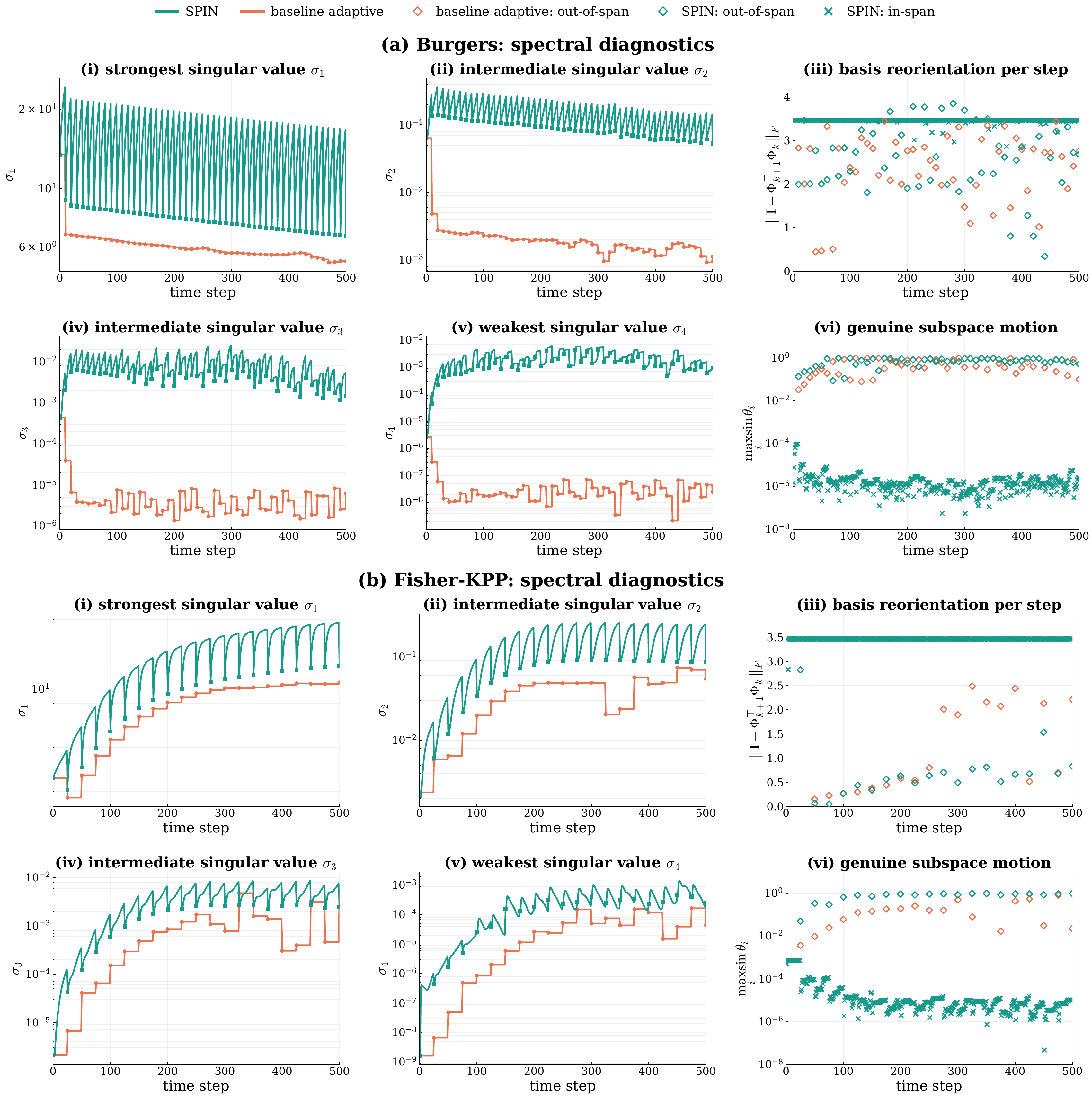}
\caption{
\textbf{The algebraic fingerprints of in-span learning in nonlinear ROMs.}
Diagnostics for \textbf{(a)} Burgers and \textbf{(b)} Fisher--KPP, comparing the baseline adaptive and SPIN ROMs.
Panels \textbf{(i)}, \textbf{(ii)}, \textbf{(iv)} and \textbf{(v)} show the retained singular values $\sigma_1,\ldots,\sigma_4$.
Panel \textbf{(iii)} shows basis reorientation, $\|\mathbf{I}-\Phi_{k+1}^{\top}\Phi_k\|_F$.
Panel \textbf{(vi)} shows genuine subspace motion, $\max_i\sin\theta_i$.
In-span updates continuously reshape and rotate the basis inside an essentially fixed subspace, while out-of-span corrections supply the geometric motion.
}
\label{fig:diagnostics}
\end{figure}

\section{Discussion}

Projection-based ROMs are usually described by their subspace. That description is correct, but incomplete. A reduced model also carries a basis spectrum, and in fixed-rank adaptive ROMs this spectrum shapes how future updates allocate limited modal capacity. We showed that these two objects can be adapted by different information streams. Out-of-span corrections move the subspace because they contain directions the ROM cannot currently represent, while in-span predictions reorganize the spectral state inside the existing subspace. This separation reveals a hidden degree of freedom in adaptive model reduction. We demonstrated in-span learning first on a closed-form spiral, where in-span update acts as a weak-mode suppressor, and then on Burgers and Fisher--KPP dynamics, where the same in-span channel improves hyper-reduced nonlinear ROM predictions. Across these examples, the common theme is trajectory awareness. The ROM's own predictions are not merely outputs; they are information about how the current representation is being used.

Several insights follow from this view:
\begin{itemize}[leftmargin=1.4em,itemsep=0.3em]
    \item \textbf{In-span predictions carry adaptive information.} Although they cannot move the trial subspace, they can rotate and reweight the directions inside it.
    \item \textbf{Spectral preconditioning is memory-controlled.} In a low-memory regime, it can suppress inactive directions and open modal capacity, as in the spiral. In high-memory regimes, it can instead reinforce several retained modes while still making the basis trajectory-informed.
    \item \textbf{In-span and out-of-span learning are complementary.} External corrections supply missing directions, while in-span updates prepare the spectral state on which those corrections act.
    \item \textbf{The gains are obtained without additional full-order queries.} The SPIN ROM uses the same external correction schedule as the baseline adaptive model, but changes how the available information is organized and used.
\end{itemize}

These results also clarify where the present framework can be improved. External information remains necessary, because in-span updates cannot create directions absent from the current reduced subspace and cannot rescue a ROM whose trajectory is already a poor guide to the local dynamics. The correction interval and forgetting factors are also decided before the simulation and remain fixed during the online prediction, whereas future methods could choose when to request out-of-span information and how strongly to forget based on online indicators such as residual growth or spectral drift. Higher-dimensional flows, multiphysics systems, and noisy or partial observations remain important research directions to investigate.

We can also draw a connection between in-span learning and in-context learning in large language models~\cite{brown2020language,xie2022explanation,akyurek2023learning}. In the linear-attention and fast-weight view of sequence models~\cite{katharopoulos2020linear,schlag2021fastweights}, this correspondence can be made mathematically explicit. Specifically, attention layers maintain a state matrix updated by rank-one outer products of internally generated representations, often discounted by a gate, and use that state to adapt their outputs without changing trained weights. The in-span update has the same basic skeleton. It accumulates a state $C$ from outer products $a_k a_k^\top$ of the model's own reduced coordinates, discounts that state by $\gamma_{\rm in}^2$, and uses it to reorganize a fixed representation. The shared structure is test-time adaptation through a decaying outer-product accumulator over self-generated representations. Adaptive ROMs may therefore provide a low-rank dynamical-systems setting in which to study how a fixed representation is reorganized by its own context. A detailed mathematical analysis linking in-span learning to in-context learning is provided in \cref{app:icl}.

More broadly, in-span learning extends beyond linear ROMs. It suggests that linear ROMs, nonlinear manifold models, and scientific machine-learning surrogates in general may all carry useful self-generated information in their evolving trajectories, creating an unexploited route toward more adaptive computational models.

\section{Methods} \label{sec:methods}

\subsection{Full-order models and time discretization} \label{sec:methods_fom}

We consider three dynamical systems. The first is a closed-form spiral in $\mathbb{R}^3$, providing a detailed and transparent view of the mechanism. The second and third are nonlinear partial differential equations, namely viscous Burgers dynamics and Fisher--KPP reaction--diffusion, used to test the same mechanism in more practical settings.

At the fully discrete level, the full-order model is written as a nonlinear time-stepping problem for a state $q^n\in\mathbb{R}^N$ at time $t_n$,
\begin{equation}
    R^n(q^n;q^{n-1}) = 0,
    \label{eq:fom_residual_general}
\end{equation}
where $R^n:\mathbb{R}^N\rightarrow\mathbb{R}^N$ denotes the residual of the chosen time discretization. For the PDE benchmarks, we use implicit backward Euler,
\begin{equation}
    R^n(q^n;q^{n-1})
    \coloneqq
    q^n-q^{n-1}-\Delta t\, f(q^n),
    \label{eq:backward_euler_residual}
\end{equation}
where $f$ is the semi-discrete spatial operator. The nonlinear system \eqref{eq:backward_euler_residual} is solved by Newton iterations in the full-order model (FOM) and by a sampled least-squares problem in the ROM.

The spiral test case is defined by the exact trajectory
\begin{equation}
    q(t)\coloneqq
    \begin{bmatrix}
    \cos t\\
    \sin t\\
    \alpha t
    \end{bmatrix},
    \qquad
    \alpha=0.4 .
    \label{eq:methods_spiral_trajectory}
\end{equation}
This trajectory solves the affine linear system
\begin{equation}
    \dot{q}=Aq+b,
    \qquad
    A\coloneqq
    \begin{bmatrix}
    0 & -1 & 0\\
    1 & 0 & 0\\
    0 & 0 & 0
    \end{bmatrix},
    \qquad
    b\coloneqq
    \begin{bmatrix}
    0\\
    0\\
    \alpha
    \end{bmatrix}.
    \label{eq:methods_spiral_ode}
\end{equation}
For time step $\Delta t=0.35$, the exact discrete map is
\begin{equation}
    q^{k+1}=Gq^k+c,
    \qquad
    G\coloneqq\exp(A\Delta t),
    \qquad
    c\coloneqq
    \begin{bmatrix}
    0\\
    0\\
    \alpha\Delta t
    \end{bmatrix}.
    \label{eq:methods_spiral_discrete}
\end{equation}
The simulation begins at $t_0=4.0$. Because Eq.~\eqref{eq:methods_spiral_discrete} is exact, the spiral experiment contains no time-integration error. All discrepancies between the FOM and ROM are therefore caused by the reduced representation and its updates.

The Burgers benchmark follows the one-dimensional viscous Burgers equation
\begin{equation}
    \partial_t u + u\,\partial_x u
    =
    \nu\,\partial_{xx}u,
    \qquad
    x\in[0,1],
    \label{eq:methods_burgers}
\end{equation}
with periodic boundary conditions and viscosity $\nu=10^{-2}$. The spatial domain is discretized with $N_x=1000$ uniform grid points. The convective term is approximated by a first-order upwind finite-difference scheme and the diffusive term by a centred second-order method. Time integration uses backward Euler with $\Delta t=10^{-3}$ over a horizon of $500$ time steps to propagate a Gaussian initial condition.

The Fisher--KPP benchmark is
\begin{equation}
    \partial_t u
    =
    D\,\partial_{xx}u+\beta u(1-u),
    \qquad
    x\in[0,1],
    \label{eq:methods_fisher}
\end{equation}
with periodic boundary conditions, diffusion coefficient $D=10^{-4}$, and reaction rate $\beta=20$. The spatial domain is discretized with $N_x=256$ uniform grid points using centred finite differences. Time integration again uses backward Euler with $\Delta t=10^{-3}$ over $500$ time steps. Starting from a localized pulse, the solution develops travelling fronts and then saturates toward $u\approx 1$ over much of the domain.

\subsection{Reduced-order representation} \label{sec:methods_rom}

All ROMs use an affine, scaled trial representation of the full-order state,
\begin{equation}
    q^n
    \approx
    q_{\mathrm{ref}} + D^{-1}\Phi a^n,
    \label{eq:methods_rom_ansatz}
\end{equation}
where $q_{\mathrm{ref}}\in\mathbb{R}^N$ is a reference state, $D\in\mathbb{R}^{N\times N}$ is a nonsingular diagonal scaling matrix, $\Phi\in\mathbb{R}^{N\times r}$ is an orthonormal basis, and $a^n\in\mathbb{R}^r$ is the reduced coordinate vector. The rank satisfies $r\ll N$.

The initial basis is constructed from a short training snapshot matrix
\begin{equation}
    X_{\mathrm{train}}
    \coloneqq
    \begin{bmatrix}
    q^1-q_{\mathrm{ref}} &
    q^2-q_{\mathrm{ref}} &
    \cdots &
    q^{N_s}-q_{\mathrm{ref}}
    \end{bmatrix},
    \label{eq:methods_training_matrix}
\end{equation}
where $N_s$ is the number of training snapshots. After scaling, we compute
\begin{equation}
    D X_{\mathrm{train}} = U\Sigma V^\top,
    \label{eq:methods_pod_svd}
\end{equation}
and set $\Phi$ equal to the first $r$ left singular vectors. In all PDE experiments reported in the main text, the offline window contains only four FOM snapshots and the ROM rank is $r=4$. The static ROM therefore represents a deliberately minimal offline model, while the adaptive ROMs are tested in a genuinely extrapolative setting.

For the spiral, the dynamics are linear and the discrete flow is known exactly, so we use a simple Galerkin projection. The initial rank-2 basis $\Phi_0$ is obtained from the two snapshots $q^1$ and $q^2$. The reduced state satisfies
\begin{equation}
    a^{n+1}=A_r a^n+c_r,
    \qquad
    A_r\coloneqq\Phi_0^\top G\Phi_0,
    \qquad
    c_r\coloneqq\Phi_0^\top c,
    \label{eq:methods_spiral_rom}
\end{equation}
with initial coordinate $a^2=\Phi_0^\top D(q^2-q_{\mathrm{ref}})$. The corresponding ROM state is
\begin{equation}
    q^n_{\mathrm{ROM}}\coloneqq q_{\mathrm{ref}} + D^{-1}\Phi_0 a^n .
    \label{eq:methods_spiral_rom_state}
\end{equation}
Note that in the spiral results, we take a zero reference state $q_{\mathrm{ref}}$ and an identity scaling matrix $D$ for simplicity. Since the basis is rank two, the ROM trajectory remains in the initial plane
\begin{equation}
    \mathcal{S}_0\coloneqq\mathrm{span}(\Phi_0).
    \label{eq:methods_initial_plane}
\end{equation}

For Burgers and Fisher--KPP, since the full-order models utilize implicit temporal integration, we use a least-squares Petrov--Galerkin (LSPG)~\cite{carlberg2011efficient, carlberg2017galerkin} formulation at the fully discrete level. Substituting the reduced ansatz \eqref{eq:methods_rom_ansatz} into the FOM residual gives the reduced least-squares problem
\begin{equation}
    a^n
    =
    \underset{\widehat{a}\in\mathbb{R}^r}{\arg\min}
    \left\|
    D\,R^n
    \left(
    q_{\mathrm{ref}}+D^{-1}\Phi\widehat{a};
    q^{n-1}_{\mathrm{ROM}}
    \right)
    \right\|_2^2 .
    \label{eq:methods_lspg_full}
\end{equation}
Here $q^{n-1}_{\mathrm{ROM}}\coloneqq q_{\mathrm{ref}}+D^{-1}\Phi a^{n-1}$ is the previous ROM state. This LSPG formulation minimizes the fully discrete residual and is well-suited to implicit time stepping.

To reduce the online cost, we use QDEIM hyper-reduction~\cite{chaturantabut2010nonlinear, drmac2016new}. Let $S\in\mathbb{R}^{N\times m}$ be a sampling matrix that selects $m$ residual entries, with $m\ll N$. The sampled LSPG problem is
\begin{equation}
    a^n
    =
    \underset{\widehat{a}\in\mathbb{R}^r}{\arg\min}
    \left\|
    S^\top D\,R^n
    \left(
    q_{\mathrm{ref}}+D^{-1}\Phi\widehat{a};
    q^{n-1}_{\mathrm{ROM}}
    \right)
    \right\|_2^2 .
    \label{eq:methods_lspg_sampled}
\end{equation}
The sample indices are chosen by QDEIM through a column-pivoted QR factorization of $\Phi^\top$. In the PDE experiments, we use $m=4$ sample points.

\subsection{Information channels in an adaptive ROM} \label{sec:methods_info}

The adaptive ROM receives two kinds of snapshots, distinguished by their relationship to the current reduced subspace.

\begin{definition}[Out-of-span and in-span snapshots]
\label{def:out_in_span_snapshots}
Let $\Phi\in\mathbb{R}^{N\times r}$ be the current orthonormal basis, and let $y\in\mathbb{R}^N$ be a preprocessed snapshot passed to a basis update. Define the projection coefficient, residual, and residual norm by
\begin{equation}
    w\coloneqq\Phi^\top y,
    \qquad
    r_y\coloneqq y-\Phi w,
    \qquad
    \rho\coloneqq\|r_y\|_2 .
    \label{eq:methods_snapshot_decomposition}
\end{equation}
The snapshot is \emph{in-span} if $\rho=0$, equivalently $y\in\mathrm{span}(\Phi)$. It is \emph{out-of-span} if $\rho>0$.
\end{definition}

In the algorithms studied here, out-of-span snapshots are generated by occasional full-order operator queries. These snapshots may contain components not represented by the current basis and can therefore move the trial subspace. In-span snapshots are the ROM's own predicted states. They satisfy $y=\Phi a$ by construction and therefore cannot introduce a new direction. The next Methods subsections show that, despite this zero residual, in-span snapshots still update the spectral state of the basis through the singular values and in-plane rotations of an iSVD update.

\subsection{iSVD update with forgetting} \label{sec:methods_isvd}

The online basis update used throughout the paper is an iSVD with forgetting. The update maintains a rank-$r$ approximation of an exponentially weighted snapshot history.

\begin{definition}[Truncated iSVD update with forgetting~\cite{bunch1978updating,brand2002incremental,brand2006fast}]
\label{def:truncated_isvd}
Let $\Phi\in\mathbb{R}^{N\times r}$ be an orthonormal basis, let $\Sigma=\mathrm{diag}(\sigma_1,\ldots,\sigma_r)$ contain the retained singular values associated with $\Phi$, and let $y\in\mathbb{R}^N$ be a preprocessed incoming snapshot. Define the projection coefficient, residual, and residual norm as in Eq.~\eqref{eq:methods_snapshot_decomposition}.
If $\rho>0$, set $\widehat{r}\coloneqq r_y/\rho$. For a forgetting factor $\gamma\in(0,1]$, form the core matrix
\begin{equation}
    K\coloneqq
    \begin{bmatrix}
    \gamma\Sigma & w\\
    0^\top & \rho
    \end{bmatrix}
    \in\mathbb{R}^{(r+1)\times(r+1)} .
    \label{eq:isvd_core}
\end{equation}
Let
\begin{equation}
    K=U_K\Sigma_K V_K^\top
    \label{eq:isvd_core_svd}
\end{equation}
be its singular-value decomposition, with singular values in non-increasing order, obtained from the eigendecomposition of the full core covariance
\begin{equation}
    KK^\top
    =
    \begin{bmatrix}
    \gamma^2\Sigma^2+ww^\top & \rho w\\
    \rho w^\top & \rho^2
    \end{bmatrix}
    =
    U_K\Sigma_K^2U_K^\top .
    \label{eq:isvd_core_covariance}
\end{equation}
 The rank-$r$ update is
\begin{equation}
    \Phi^+ =
    \begin{bmatrix}
    \Phi & \widehat{r}
    \end{bmatrix}
    U_K(:,1:r),
    \qquad
    \Sigma^+ = \Sigma_K(1:r,1:r).
    \label{eq:isvd_update}
\end{equation}
If $\rho=0$, the residual direction is not defined and the update is interpreted through the reduced core described later in \cref{thm:inspan_update}.
\end{definition}

The update in \cref{def:truncated_isvd} is equivalent to appending the new snapshot to a weighted history matrix.
Let $Y_k$ denote the exponentially weighted data matrix containing the snapshots assimilated up to update step $k$, \[ Y_k \coloneqq \begin{bmatrix} \gamma^{k-1}y_1 & \gamma^{k-2}y_2 & \cdots & \gamma y_{k-1} & y_k \end{bmatrix}. \] The current rank-$r$ representation approximates this weighted data matrix by $Y_k\approx \Phi\Sigma V^\top$. Receiving a new snapshot $y$ gives the augmented weighted history \[ Y_{k+1} = \begin{bmatrix} \gamma Y_k & y \end{bmatrix}. \]

The small matrix $K$ in Eq.~\eqref{eq:isvd_core} contains all information needed to update the left singular vectors of this augmented matrix. The forgetting factor $\gamma$ discounts previous information before the new snapshot is added, and acts as a memory valve in this sense.

Equivalently, the update tracks the dominant eigenspace of an exponentially weighted second-moment matrix. If
$
    M_k \coloneqq Y_kY_k^\top
$
denotes the weighted second-moment matrix in the ambient coordinates, then the weighted-history update gives
\begin{equation}
    M_{k+1}
    =
    \gamma^2 M_k + yy^\top .
    \label{eq:weighted_covariance_full}
\end{equation}

\subsection{In-span learning preserves the subspace and reorganizes the spectrum} \label{sec:methods_isl}

We now specialize \cref{def:truncated_isvd} to the case in which the incoming snapshot is a ROM prediction. In scaled coordinates, such a snapshot has the form
\begin{equation}
    y=\Phi a,
    \qquad
    a\in\mathbb{R}^r .
    \label{eq:inspan_snapshot}
\end{equation}
Hence, according to Eq.~\eqref{eq:methods_snapshot_decomposition}, $w=a$, $r_y=0$, and $\rho=0$. No new ambient direction can be added. Nevertheless, the singular values and the coordinate frame inside the subspace can change.

\begin{theorem}[Subspace-preserving spectral preconditioning by in-span learning]
\label{thm:inspan_update}
Let $\Phi\in\mathbb{R}^{N\times r}$ be orthonormal, let $\Sigma=\mathrm{diag}(\sigma_1,\ldots,\sigma_r)$, and let $y=\Phi a$ be an in-span snapshot. Applying the iSVD update with forgetting factor $\gamma_{\rm in}\in(0,1]$ gives
\begin{equation}
    \Phi^+=\Phi U,
    \qquad
    \Sigma^+
    =
    \Lambda^{1/2},
    \label{eq:inspan_update_result}
\end{equation}
where $U\in\mathbb{R}^{r\times r}$ is orthogonal and
\begin{equation}
    U\Lambda U^\top
    =
    \gamma_{\rm in}^2\Sigma^2 + aa^\top .
    \label{eq:inspan_reduced_covariance}
\end{equation}
Consequently,
\begin{equation}
    \mathrm{span}(\Phi^+)=\mathrm{span}(\Phi).
    \label{eq:inspan_span_preserved}
\end{equation}
The singular values obey the rank-one interlacing inequalities
\begin{equation}
    \sigma_1^+
    \geq
    \gamma_{\rm in}\sigma_1
    \geq
    \sigma_2^+
    \geq
    \cdots
    \geq
    \gamma_{\rm in}\sigma_{r-1}
    \geq
    \sigma_r^+
    \geq
    \gamma_{\rm in}\sigma_r .
    \label{eq:inspan_interlacing}
\end{equation}
\end{theorem}

\begin{proof}
For an in-span snapshot, $y=\Phi a$. Therefore $w=\Phi^\top y=a$ and $\rho=0$. The core matrix in Eq.~\eqref{eq:isvd_core} has zero last row,
\begin{equation}
    K=
    \begin{bmatrix}
    \gamma_{\rm in}\Sigma & a\\
    0^\top & 0
    \end{bmatrix}.
\end{equation}
Its nonzero singular values are the singular values of the reduced matrix
\begin{equation}
    K_{\mathrm{in}}=
    \begin{bmatrix}
    \gamma_{\rm in}\Sigma & a
    \end{bmatrix}
    \in\mathbb{R}^{r\times(r+1)} .
\end{equation}
Since
\begin{equation}
    K_{\mathrm{in}}K_{\mathrm{in}}^\top
    =
    \gamma_{\rm in}^2\Sigma^2+aa^\top ,
\end{equation}
the left singular vectors of $K_{\mathrm{in}}$ are the eigenvectors of the reduced covariance in Eq.~\eqref{eq:inspan_reduced_covariance}. The updated basis is therefore $\Phi^+=\Phi U$ for an orthogonal matrix $U$, which proves subspace preservation. The interlacing inequalities follow from Cauchy's interlacing theorem for a positive semidefinite rank-one perturbation of the diagonal matrix $\gamma_{\rm in}^2\Sigma^2$.
\end{proof}

After $s$ consecutive in-span updates with reduced coordinates $a^{n_1},\ldots,a^{n_s}$, the reduced covariance is
\begin{equation}
    C_s
    =
    \gamma_{\rm in}^{2s}\Sigma_0^2
    +
    \sum_{j=1}^{s}
    \gamma_{\rm in}^{2(s-j)}
    a^{n_j}(a^{n_j})^\top .
    \label{eq:inspan_covariance_after_s}
\end{equation}
This is the covariance recursion used in the spiral example. It shows explicitly how the ROM trajectory reshapes the spectrum. Directions aligned with recent coefficient vectors can receive repeated rank-one reinforcement, while directions receiving little trajectory energy can be suppressed by forgetting. Here, we say that a direction is \emph{reinforced} if its post-update directional energy exceeds its pre-update value, and \emph{suppressed} if its post-update directional energy is smaller.

\begin{definition}[Directional reinforcement and suppression]
\label{def:directional_reinforcement}
Let $\xi\in\mathbb{R}^r$ be a unit direction in the current, pre-update reduced coordinates. For an in-span update with reduced coordinate $a$, define the pre-update and post-update directional energies by
\begin{equation}
    E_\xi\coloneqq\xi^\top\Sigma^2\xi,
    \qquad
    E_\xi^+\coloneqq\xi^\top
    \left(
    \gamma_{\rm in}^2\Sigma^2+aa^\top
    \right)
    \xi .
    \label{eq:directional_energies}
\end{equation}
We say that the fixed pre-update direction $\xi$ is \emph{reinforced} if $E_\xi^+>E_\xi$ and \emph{suppressed} if $E_\xi^+<E_\xi$.
\end{definition}

Using Eq.~\eqref{eq:directional_energies},
\begin{equation}
    E_\xi^+-E_\xi
    =
    (\xi^\top a)^2
    -
    (1-\gamma_{\rm in}^2)E_\xi .
    \label{eq:directional_energy_change}
\end{equation}
Therefore the fixed pre-update direction $\xi$ is reinforced if
\begin{equation}
    |\xi^\top a|
    >
    \sqrt{1-\gamma_{\rm in}^2}\,
    \sqrt{E_\xi},
    \label{eq:directional_reinforcement_threshold}
\end{equation}
and suppressed if the inequality is reversed. In particular, for a pre-update singular-vector direction $e_i$,
\begin{equation}
    E_{e_i}^+-E_{e_i}
    =
    a_i^2-(1-\gamma_{\rm in}^2)\sigma_i^2,
    \label{eq:coordinate_energy_change}
\end{equation}
so that this pre-update direction is reinforced when
\begin{equation}
    |a_i|
    >
    \sqrt{1-\gamma_{\rm in}^2}\,\sigma_i,
    \label{eq:coordinate_reinforcement_threshold}
\end{equation}
and suppressed when $|a_i|$ falls below this threshold.

This analysis also shows that in-span spectral preconditioning does not always appear as suppression of the weakest mode. If the ROM trajectory has enough projected energy in a direction to exceed the threshold, that direction is reinforced even under forgetting. If the projected energy falls below the threshold, the direction decays. The spiral example with aggressive forgetting realizes the latter case for one nearly inactive direction, whereas the PDE examples show richer regimes with less forgetting in which several retained directions are reinforced while the basis still rotates and reweights.

\subsection{Full-order operator query} \label{sec:methods_query}

The out-of-span channel requires occasional information that is not constrained to the current reduced subspace. In this work, that information is supplied by a full-order operator query. The idea is to let the ROM advance the system most of the time, and to call the full-order model only at selected correction events.

Suppose an out-of-span event is triggered at time step $n+1$. Before the correction, the ROM has produced the state
\begin{equation}
    q^n_{\mathrm{ROM}}
    \coloneqq
    q_{\mathrm{ref}}+D^{-1}\Phi a^n .
    \label{eq:rom_state_before_correction}
\end{equation}
We then initialize one full-order step from this ROM state and solve
\begin{equation}
    R^{n+1}
    \left(
    q^{n+1}_{\mathrm{corr}};
    q^n_{\mathrm{ROM}}
    \right)
    =
    0 .
    \label{eq:rom_initialized_correction}
\end{equation}
The resulting state $q^{n+1}_{\mathrm{corr}}$ is used as the external correction snapshot. In scaled coordinates,
\begin{equation}
    y_{\mathrm{corr}}
    \coloneqq
    D\left(q^{n+1}_{\mathrm{corr}}-q_{\mathrm{ref}}\right).
    \label{eq:scaled_correction_snapshot}
\end{equation}
This snapshot is then passed to the out-of-span iSVD update with forgetting factor $\gamma_{\rm out}$.

This correction is different from a coarse full-order lookahead trajectory previously used in~\cite{hedayat2026history}. A coarse-FOM oracle advances an auxiliary full-order state on a larger time step and uses that state as the correction signal. The full-order operator query used here instead takes a single fine-grid full-order step from the current ROM state. We use it because it gives a simple hybrid model, where the reduced model carries the trajectory most of the time and occasionally interacts with the full-order operator to receive a local out-of-span correction snapshot.

For the spiral example, Eq.~\eqref{eq:rom_initialized_correction} reduces to one exact full-order map applied to the ROM state,
\begin{equation}
    q_{\mathrm{corr}}
    \coloneqq
    G q^n_{\mathrm{ROM}}+c .
    \label{eq:spiral_rom_initialized_correction_methods}
\end{equation}
For the PDE benchmarks, Eq.~\eqref{eq:rom_initialized_correction} is solved with the same backward-Euler residual used by the FOM, but only at the correction times.

\subsection{SPIN algorithm} \label{sec:methods_alg}

The complete adaptive procedure (\cref{alg:inspan_outspan_rom}) alternates between three operations. First, the ROM advances one time step using the current basis and sample set. Second, if the step is not an external correction time, the ROM-predicted state is used for an in-span update with forgetting factor $\gamma_{\rm in}$. Third, every $z_s$ steps, a full-order operator query is performed, and the produced snapshot is used for an out-of-span update with forgetting factor $\gamma_{\rm out}$.

After any basis update, basis-dependent quantities must be refreshed. For the LSPG--QDEIM ROMs, this includes the sampled basis and its pseudoinverse. After an out-of-span update, the QDEIM sample indices are recomputed. After an in-span update, the subspace is unchanged but the basis vectors have rotated inside it, so the sampled operators depending on the basis are refreshed. The reduced coordinates are also transferred to the updated basis by projecting the current reconstructed ROM state onto the new basis.

\begin{algorithm}[!t]
\caption{SPIN ROM algorithm}
\label{alg:inspan_outspan_rom}
\begin{algorithmic}[1]
\Require Initial basis $\Phi$, singular values $\sigma$, reference state $q_{\mathrm{ref}}$, scaling matrix $D$, sample matrix $S$, reduced state $a^0$, update interval $z_s$, forgetting factors $\gamma_{\rm in}$ and $\gamma_{\rm out}$, final step $N_t$
\For{$n=0,1,\ldots,N_t-1$}
    \State Advance the sampled ROM one step:
    \[
        a^{n+1}\leftarrow \mathrm{ROM\_Step}(\Phi,S,a^n)
    \]
    \State Reconstruct the ROM state:
    \[
        q^{n+1}_{\mathrm{ROM}}
        \leftarrow
        q_{\mathrm{ref}}+D^{-1}\Phi a^{n+1}
    \]
    \If{$(n+1)\bmod z_s=0$}
        \Comment{\textcolor{outspan}{\textbf{out-of-span adaptation}}}
        \State Perform one full-order operator query:
        \[
            R^{n+1}
            \left(q^{n+1}_{\mathrm{corr}};q^n_{\mathrm{ROM}}\right)=0
        \]
        \State Preprocess the correction:
        \[
            y \leftarrow D(q^{n+1}_{\mathrm{corr}}-q_{\mathrm{ref}})
        \]
        \State Update the basis and singular values by out-of-span iSVD:
        \[
            (\Phi,\sigma)\leftarrow
            \mathrm{iSVD}(\Phi,\sigma,y,\gamma_{\rm out})
        \]
        \State Recompute QDEIM sample matrix $S$ and sampled operators.
    \Else
        \Comment{\textcolor{inspan}{\textbf{in-span adaptation}}}
        \State Preprocess the ROM prediction:
        \[
            y \leftarrow D(q^{n+1}_{\mathrm{ROM}}-q_{\mathrm{ref}})
        \]
        \State Update the basis and singular values by in-span iSVD:
        \[
            (\Phi,\sigma)\leftarrow
            \mathrm{iSVD}(\Phi,\sigma,y,\gamma_{\rm in})
        \]
        \State Refresh sampled operators.
    \EndIf
    \State Transfer the current state to the updated basis:
    \[
        a^{n+1}\leftarrow
        \Phi^\top D(q^{n+1}_{\mathrm{ROM}}-q_{\mathrm{ref}})
    \]
\EndFor
\end{algorithmic}
\end{algorithm}

The baseline adaptive ROM used for comparison can be obtained from \cref{alg:inspan_outspan_rom} by omitting the in-span update in the \texttt{else} branch. The static ROM is obtained by omitting both in-span and out-of-span updates. The implementation used in this work is available at \url{https://github.com/APHedayat/SPIN}.

\subsection{Metrics and diagnostics} \label{sec:methods_diagnostics}

We use two groups of diagnostics. The first group measures how a single out-of-span correction is absorbed by a fixed-rank update. Let $\Phi_{\mathrm{pre}}$ and $\Phi_{\mathrm{post}}$ denote the orthonormal bases immediately before and after the correction update, and let $q_{\mathrm{corr}}$ be the correction snapshot. The out-of-span residual relative to the pre-update subspace is
\begin{equation}
    r
    \coloneqq
    q_{\mathrm{corr}}
    -
    \Phi_{\mathrm{pre}}\Phi_{\mathrm{pre}}^\top q_{\mathrm{corr}},
    \qquad
    \hat r
    \coloneqq
    r/\|r\|_2 .
    \label{eq:diagnostic_correction_residual}
\end{equation}
The residual capture is
\begin{equation}
    \chi_{\mathrm{res}}
    \coloneqq
    \|\Phi_{\mathrm{post}}^\top \hat r\|_2 ,
    \label{eq:diagnostic_residual_capture}
\end{equation}
which measures how much of the new residual direction lies in the post-update subspace. The plane-change angle is the largest principal angle between the pre- and post-update subspaces,
\begin{equation}
    \theta_{\mathrm{plane}}
    \coloneqq
    \max_i \arccos\!\left(
    \zeta_i(\Phi_{\mathrm{pre}}^\top\Phi_{\mathrm{post}})
    \right),
    \label{eq:diagnostic_plane_angle}
\end{equation}
where $\zeta_i(\cdot)$ are singular values. The correction error is
\begin{equation}
    e_{\mathrm{corr}}
    \coloneqq
    \left\|
    q_{\mathrm{corr}}
    -
    \Phi_{\mathrm{post}}\Phi_{\mathrm{post}}^\top q_{\mathrm{corr}}
    \right\|_2 .
    \label{eq:diagnostic_correction_error}
\end{equation}
This measures the part of the correction snapshot not represented by the updated subspace.

The second group of diagnostics measures how the adaptive basis, spectrum, and subspace evolve over the online trajectory. The first trajectory diagnostic is the singular-value history,
\begin{equation}
    \sigma_1^n,\ldots,\sigma_r^n .
    \label{eq:diagnostic_singular_values}
\end{equation}
These values measure how the adaptive algorithm weights the retained modes. In the baseline adaptive ROM, they change only at correction events. In the SPIN ROM, they may change at every time step.

The second trajectory diagnostic measures basis reorientation between consecutive steps,
\begin{equation}
    d_{\mathrm{basis}}^n
    \coloneqq
    \left\|
    I-\left(\Phi^{n+1}\right)^\top\Phi^n
    \right\|_F .
    \label{eq:diagnostic_basis_rotation_methods}
\end{equation}
This quantity is sensitive to changes in the orthonormal frame. It can be nonzero even when the two bases span the same subspace.

The third trajectory diagnostic measures genuine subspace change. Let the singular values of
\begin{equation}
    \left(\Phi^n\right)^\top\Phi^{n+1}
\end{equation}
be $s_1,\ldots,s_r$, with $s_i=\cos\theta_i$, where $\theta_i$ are the principal angles between the two subspaces. We report
\begin{equation}
    d_{\mathrm{subspace}}^n
    \coloneqq
    \max_i \sin\theta_i
    =
    \sqrt{1-\min_i s_i^2}.
    \label{eq:diagnostic_principal_angle_methods}
\end{equation}
This diagnostic is zero when the subspace is unchanged, even if the basis has rotated inside that subspace. The combination of nonzero $d_{\mathrm{basis}}^n$ and zero $d_{\mathrm{subspace}}^n$ is therefore the numerical signature of a pure in-span update.

\subsection{Computational complexity} \label{sec:methods_complexity}

The additional cost of in-span learning is small compared with the cost of obtaining out-of-span information from the full-order operator. Let $N$ be the full-order dimension, $r$ the ROM rank, and $m$ the number of sampled residual entries, with $r,m\ll N$.

We first define the cost of one hyper-reduced LSPG time step. A sampled LSPG step solves a nonlinear least-squares problem in $r$ unknowns using only $m$ sampled residual entries. Let $\mathcal{C}_{\mathrm{eval}}^{(m)}$ denote the cost of evaluating the sampled residual and sampled Jacobian contributions required by one Newton iteration. Once these sampled quantities are available, the dense reduced linear algebra scales as
\begin{equation}
    \mathcal{C}_{\mathrm{LSPG}}
    \coloneqq
    \mathcal{C}_{\mathrm{eval}}^{(m)}
    +
    \mathcal{O}(mr^2+r^3)
    \label{eq:sampled_lspg_cost}
\end{equation}
per Newton iteration. Thus, between correction events, the online simulation remains reduced and sampled.

The cost of an iSVD basis update is roughly similar for in-span and out-of-span snapshots. For a general out-of-span snapshot, the update projects the snapshot onto the current basis, forms the orthogonal residual, computes the SVD of a small $(r+1)\times(r+1)$ core matrix, and rotates the ambient basis. The dominant scaling is
\begin{equation}
    \mathcal{C}_{\mathrm{iSVD}}^{\mathrm{out}}
    \coloneqq
    \mathcal{O}(Nr^2+r^3).
    \label{eq:isvd_out_cost}
\end{equation}
For an in-span snapshot, the orthogonal residual is zero and the update can be performed through the reduced covariance in Eq.~\eqref{eq:inspan_reduced_covariance}. This avoids the residual construction, but the small eigendecomposition and the ambient basis rotation remain. Therefore, the in-span update costs as
\begin{equation}
    \mathcal{C}_{\mathrm{iSVD}}^{\mathrm{in}}
    \coloneqq
    \mathcal{O}(Nr^2+r^3).
    \label{eq:isvd_in_cost}
\end{equation}
The in-span update can be cheaper in implementations that delay the ambient rotation and accumulate the transformation in reduced coordinates, but in the implementation considered for PDE cases the leading order is the same as for an out-of-span iSVD update.

Basis changes may also require refreshing hyper-reduction quantities. We denote this cost by $\mathcal{C}_{\mathrm{refresh}}$. If QDEIM sampling points are recomputed from a pivoted QR factorization of $\Phi^\top$, then
\begin{equation}
    \mathcal{C}_{\mathrm{QDEIM}}
    \coloneqq
    \mathcal{O}(Nr^2).
    \label{eq:qdeim_cost}
\end{equation}
In-span updates preserve the trial subspace, so one may keep the same sampling set between out-of-span corrections and update only the sampled basis-dependent quantities. If the sampling points are recomputed after every basis rotation, then the same $\mathcal{C}_{\mathrm{QDEIM}}$ cost applies.

The bottleneck in the adaptive procedure is the full-order operator query used to obtain an out-of-span correction. For implicit PDE solvers, this query requires nonlinear iterations, residual and Jacobian evaluations, and large linear solves. We write its cost as
\begin{equation}
    \mathcal{C}_{\mathrm{FOM}}
    \coloneqq
    \mathcal{O}(N^\beta),
    \qquad \beta>1,
    \label{eq:fom_query_cost}
\end{equation}
where $\beta$ depends on the discretization, nonlinear solver, linear solver, preconditioner, and sparsity structure. This term is typically much larger than the reduced dense linear algebra when $N$ is large and $r,m\ll N$.

With these ingredients, the cost of one in-span update is
\begin{equation}
    \mathcal{C}_{\mathrm{in}}
    \coloneqq
    \mathcal{C}_{\mathrm{iSVD}}^{\mathrm{in}}
    +
    \mathcal{C}_{\mathrm{refresh}}^{\mathrm{in}},
    \label{eq:inspan_event_cost}
\end{equation}
where $C_{\mathrm{refresh}}^{\mathrm{in}}$ denotes any update of sampled basis-dependent quantities performed after the in-span rotation. The cost of one out-of-span correction event is
\begin{equation}
    \mathcal{C}_{\mathrm{out}}
    \coloneqq
    \mathcal{C}_{\mathrm{FOM}}
    +
    \mathcal{C}_{\mathrm{iSVD}}^{\mathrm{out}}
    +
    \mathcal{C}_{\mathrm{QDEIM}}
    +
    \mathcal{C}_{\mathrm{refresh}}^{\mathrm{out}},
    \label{eq:outspan_event_cost}
\end{equation}
where $\mathcal{C}_{\mathrm{refresh}}^{\mathrm{out}}$ denotes the remaining basis-dependent reduced-operator or sampled-operator updates.

The key comparison is that the SPIN ROM uses the same out-of-span correction schedule as the baseline adaptive ROM. It therefore does not add any $\mathcal{C}_{\mathrm{FOM}}$ terms. If the correction interval is $z_s$, then between two out-of-span corrections the SPIN model adds
\begin{equation}
    (z_s-1)\mathcal{C}_{\mathrm{in}}
\end{equation}
relative to the baseline adaptive model, but it does not change the number of full-order operator queries. In the large-scale implicit PDE regime, $\mathcal{C}_{\mathrm{FOM}}=\mathcal{O}(N^\beta)$ is the dominant cost, whereas the in-span operations are reduced or low-rank dense linear algebra.

\section*{Data availability}
The data supporting the findings of this work are available from the corresponding author upon request.

\section*{Code availability}
The code used to generate the numerical results in this work is available at \url{https://github.com/APHedayat/SPIN}.

\section*{Acknowledgements}

\textbf{A.H.} and \textbf{K.D.} were supported by the Office of Under Secretary of Defense for Research and Engineering (OUSD(RE)) grant N00014-21-1-295. \textbf{L.B.} acknowledges NSF award CCF-2331590 and ARO award W911NF-26-1-A219.

\section*{Author contributions}

\textbf{A.H.:} Methodology, Software, Investigation, Data Curation, Writing - Original Draft, Visualization.
\textbf{L.B.:} Conceptualization, Methodology, Writing - Review \& Editing, Supervision.
\textbf{K.D.:} Conceptualization, Methodology, Investigation, Writing - Review \& Editing, Supervision, Funding acquisition.

\section*{Competing interests}

The authors declare no competing interests.


\appendix
%
\clearpage

\renewcommand{\thefigure}{\thesection\arabic{figure}}
\renewcommand{\thetable}{\thesection\arabic{table}}
\renewcommand{\theequation}{\thesection.\arabic{equation}}
\makeatletter
\@addtoreset{figure}{section}
\@addtoreset{table}{section}
\@addtoreset{equation}{section}
\makeatother

\section*{Appendices}
\addcontentsline{toc}{section}{Appendices}
\noindent
The following appendices contain the Supplementary Information accompanying the manuscript. They provide notation, full numerical details for the spiral example, comparisons with alternative fixed-rank updates, hyperparameter sensitivity studies for the nonlinear PDE benchmarks, reproducibility information, and a mathematical connection between in-span learning and in-context learning supporting the main text.

\section{Notation}
\label{app:notation}

\Cref{tab:supp_notation} summarizes the main symbols used in the paper.

\begin{table}[!ht]
\centering
\caption{\textbf{Notation used in the main text and appendices.}}
\label{tab:supp_notation}
\begin{tabularx}{\linewidth}{lX}
\toprule
Symbol & Meaning \\
\midrule
$N$ & Full-order dimension \\
$r$ & ROM rank \\
$m$ & Number of hyper-reduction sample points \\
$N_t$ & Number of online prediction time steps \\
$q^n$ & Full-order state at time step $n$ \\
$q^n_{\mathrm{ROM}}$ & ROM-reconstructed state at time step $n$ \\
$q_{\mathrm{corr}}$ & Correction snapshot obtained from a full-order operator query \\
$q_{\mathrm{ref}}$ & Reference state used in the affine ROM representation \\
$D$ & Diagonal scaling matrix \\
$\Phi\in\mathbb{R}^{N\times r}$ & Orthonormal reduced basis \\
$\Phi_0$ & Initial reduced basis \\
$\Phi_{\mathrm{pre}}$, $\Phi_{\mathrm{post}}$ & Reduced bases immediately before and after a correction update \\
$\Sigma=\mathrm{diag}(\sigma_1,\ldots,\sigma_r)$ & Retained singular values associated with $\Phi$ \\
$a^n\in\mathbb{R}^r$ & Reduced coordinate vector at time step $n$ \\
$S$ & Hyper-reduction sampling matrix \\
$z_s$ & Number of ROM time steps between out-of-span corrections \\
$\gamma_{\rm in}$ & Forgetting factor for in-span updates \\
$\gamma_{\rm out}$ & Forgetting factor for out-of-span updates \\
$y$ & Preprocessed snapshot passed to iSVD \\
$w=\Phi^\top y$ & In-plane coefficients of $y$ \\
$r_y=y-\Phi w$ & Orthogonal residual of $y$ with respect to $\mathrm{span}(\Phi)$ \\
$\rho=\|r_y\|_2$ & Residual norm \\
$K$ & Small iSVD core matrix \\
$Y_k$ & Exponentially weighted snapshot-history matrix after $k$ updates \\
$M_k=Y_kY_k^\top$ & Exponentially weighted second-moment matrix in ambient coordinates \\
$C$ & Reduced covariance represented in reduced coordinates \\
$C_s$ & Reduced covariance after $s$ consecutive in-span updates \\
$C_{\mathrm{corr}}$ & In-plane block of the iSVD core covariance used to diagnose local rank competition \\
$\eta=\rho^2/\lambda_r(C_{\mathrm{corr}})$ & Residual-to-weakest-mode ratio \\
$\chi_{\mathrm{res}}$ & Residual capture by the post-update basis \\
$\theta_{\mathrm{plane}}$ & Largest principal angle between the pre- and post-update subspaces \\
$e_{\mathrm{corr}}$ & Correction error after projection onto the post-update subspace \\
$d_{\mathrm{basis}}^n$ & Basis-reorientation diagnostic between steps $n$ and $n+1$ \\
$d_{\mathrm{subspace}}^n$ & Genuine subspace-motion diagnostic between steps $n$ and $n+1$ \\
$\theta_i$ & Principal angles between two subspaces \\
$\zeta_i(\cdot)$ & Singular values of a matrix argument, used in principal-angle diagnostics \\
$\mathcal{C}_{\mathrm{FOM}}$ & Cost of one full-order operator query \\
$\mathcal{C}_{\mathrm{in}}$, $\mathcal{C}_{\mathrm{out}}$ & Costs of one in-span update and one out-of-span correction event \\
\bottomrule
\end{tabularx}
\end{table}

\section{Full numerical details for the spiral example}
\label{app:spiral_details}

The spiral trajectory is
\begin{equation}
    q(t)
    =
    \begin{bmatrix}
    \cos t\\
    \sin t\\
    \alpha t
    \end{bmatrix},
    \qquad
    \alpha=0.4.
\end{equation}
The time step is $\Delta t=0.35$ and the initial time is $t_0=4.0$. The exact discrete dynamics are
\begin{equation}
    q^{k+1}=Gq^k+c,
\end{equation}
with
\begin{equation}
    G=
    \begin{bmatrix}
    0.93937271 & -0.34289781 & 0\\
    0.34289781 & 0.93937271 & 0\\
    0 & 0 & 1
    \end{bmatrix},
    \qquad
    c=
    \begin{bmatrix}
    0\\
    0\\
    0.14
    \end{bmatrix}.
\end{equation}

The initial rank-2 basis is obtained from the two snapshots $q^1$ and $q^2$:
\begin{equation}
    \Phi_0
    =
    \begin{bmatrix}
    -0.08645625 &  0.99594990\\
    -0.46958378 & -0.01888989\\
     0.87864463 &  0.08790323
    \end{bmatrix}.
    \label{eq:si_spiral_phi0}
\end{equation}
The associated singular values are
\begin{equation}
    \Sigma_0
    =
    \mathrm{diag}(2.91533718,\;0.25061752).
    \label{eq:si_spiral_sigma0}
\end{equation}

With $\gamma_{\rm in}=0.1$, four in-span updates give the preconditioned basis
\begin{equation}
    \Phi_{\mathrm{in,pre}}
    =
    \begin{bmatrix}
    -0.40529412 & -0.91385314\\
     0.42031170 & -0.21024705\\
    -0.81183419 &  0.34737389
    \end{bmatrix},
\end{equation}
and singular values
\begin{equation}
    \Sigma_{\mathrm{in,pre}}
    =
    \mathrm{diag}(2.47602980,\;0.03093119).
\end{equation}
The subspace is unchanged:
\begin{equation}
    \mathrm{span}(\Phi_{\mathrm{in,pre}})
    =
    \mathrm{span}(\Phi_0).
\end{equation}

\subsection{Full-order operator query}

The out-of-span correction is generated from the ROM state,
\begin{equation}
    q_{\mathrm{corr}}
    =
    G q^5_{\mathrm{ROM}}+c
    =
    \begin{bmatrix}
    1.29594031\\
    -0.62913467\\
    2.13998855
    \end{bmatrix}.
\end{equation}
The residual relative to the initial ROM plane is
\begin{equation}
    r
    =
    q_{\mathrm{corr}}-\Phi_0\Phi_0^\top q_{\mathrm{corr}}
    =
    \begin{bmatrix}
    -0.01029214\\
    0.36809432\\
    0.19571200
    \end{bmatrix},
\end{equation}
with
\begin{equation}
    \rho=\|r\|_2=0.41701624,
    \qquad
    \rho^2=0.17390255 .
\end{equation}

The correction coefficients in the two pre-update coordinate systems are
\begin{equation}
    \Phi_0^\top q_{\mathrm{corr}}
    =
    \begin{bmatrix}
    2.06367875\\
    1.49068781
    \end{bmatrix},
    \qquad
    \Phi_{\mathrm{in,pre}}^\top q_{\mathrm{corr}}
    =
    \begin{bmatrix}
    -2.52698551\\
    -0.30864925
    \end{bmatrix}.
\end{equation}

\subsection{Rank-2 competition}

The competition matrices are
\begin{equation}
    C_{\mathrm{corr,\,out}}
    =
    \Sigma_0^2
    +
    (\Phi_0^\top q_{\mathrm{corr}})
    (\Phi_0^\top q_{\mathrm{corr}})^\top,
\end{equation}
and
\begin{equation}
    C_{\mathrm{corr,\,in}}
    =
    \Sigma_{\mathrm{in,pre}}^2
    +
    (\Phi_{\mathrm{in,pre}}^\top q_{\mathrm{corr}})
    (\Phi_{\mathrm{in,pre}}^\top q_{\mathrm{corr}})^\top .
\end{equation}
Their weakest eigenvalues are
\begin{equation}
    \lambda_{\min}(C_{\mathrm{corr,\,out}})=1.44819350,
    \qquad
    \lambda_{\min}(C_{\mathrm{corr,\,in}})=0.04743386.
\end{equation}
Thus
\begin{equation}
    \eta_{\mathrm{out}}
    =
    \frac{\rho^2}{\lambda_{\min}(C_{\mathrm{corr,\,out}})}
    \approx
    0.12,
    \qquad
    \eta_{\mathrm{in}}
    =
    \frac{\rho^2}{\lambda_{\min}(C_{\mathrm{corr,\,in}})}
    \approx
    3.67 .
\end{equation}

\subsection{Post-update bases}

Applying the same truncated iSVD update (with $\gamma_{\rm out}=1.0$ for simplicity) to both pre-update states gives
\begin{equation}
    \Phi_{\mathrm{out,post}}
    =
    \begin{bmatrix}
    0.17789643 & -0.94264402\\
    -0.39173043 & -0.33113817\\
    0.90271819 & 0.04206862
    \end{bmatrix},
\end{equation}
and
\begin{equation}
    \Phi_{\mathrm{in,post}}
    =
    \begin{bmatrix}
    -0.45818482 & 0.48087061\\
    0.33012848 & 0.87292840\\
    -0.82527684 & 0.08221594
    \end{bmatrix}.
\end{equation}

We define the residual capture as $\|\Phi_{\mathrm{post}}^\top \hat r\|_2$, where $\hat r \coloneqq r/\|r\|_2$, the plane angle as the largest principal angle $\max_i\arccos(\zeta_i(\Phi_{\mathrm{pre}}^\top\Phi_{\mathrm{post}}))$, where $\zeta_i(\cdot)$ are singular values, the retained alignment as $\max_j |(\phi^{\mathrm{post}}_j)^\top \hat r|$, where $\phi^{\mathrm{post}}_j$ is the $j$th column of $\Phi_{\mathrm{post}}$, and the correction error as $\|q_{\mathrm{corr}}-\Phi_{\mathrm{post}}\Phi_{\mathrm{post}}^\top q_{\mathrm{corr}}\|_2$. The measured outcomes are listed in \cref{tab:supp_spiral_main}.

\begin{table}[!ht]
\centering
\caption{\textbf{Main spiral diagnostics for baseline adaptive and SPIN updates.}}
\label{tab:supp_spiral_main}
\begin{tabular}{lcccc}
\toprule
Method & Residual capture & Plane angle & Retained alignment & Correction error \\
\midrule
Baseline adaptive & 0.260 & $15.06^\circ$ & 0.249 & $1.02\times 10^{-2}$ \\
SPIN & 0.802 & $53.29^\circ$ & 0.797 & $8.84\times 10^{-4}$ \\
\bottomrule
\end{tabular}
\end{table}

\subsection{Core-matrix interpretation}

The same effect can be read directly from the small iSVD core matrix. For a pre-update basis $\Phi_{\mathrm{pre}}$ and singular values $\Sigma_{\mathrm{pre}}$, write
\[
    \alpha=\Phi_{\mathrm{pre}}^\top q_{\mathrm{corr}},
    \qquad
    r=q_{\mathrm{corr}}-\Phi_{\mathrm{pre}}\alpha,
    \qquad
    \rho=\|r\|_2,
    \qquad
    \hat r=r/\rho .
\]
The out-of-span iSVD update forms the augmented basis $[\Phi_{\mathrm{pre}}\;\;\hat r]$ and the core matrix
\begin{equation}
    K
    =
    \begin{bmatrix}
    \gamma_{\rm out}\Sigma_{\mathrm{pre}} & \alpha\\
    0^\top & \rho
    \end{bmatrix}.
\end{equation}
Here we use $\gamma_{\rm out}=1.0$. The left singular vectors of $K$ determine how the retained post-update basis vectors mix the old pre-update modes with the new residual direction. Because the update is truncated back to rank two, the first retained vector corresponds to the dominant slot, while the second retained vector is the weak slot where the old weak mode competes with the incoming residual.

For the baseline adaptive model, $\Phi_{\mathrm{pre}}=\Phi_0$ and $\Sigma_{\mathrm{pre}}=\Sigma_0$, giving
\begin{equation}
    K_{\mathrm{out}}
    =
    \begin{bmatrix}
    2.91533718 & 0 & 2.06367875\\
    0 & 0.25061752 & 1.49068781\\
    0 & 0 & 0.41701624
    \end{bmatrix}.
\end{equation}
The second retained left singular vector is
\begin{equation}
    u^{(2)}_{K,\mathrm{out}}
    =
    \begin{bmatrix}
    0.27395794\\
    -0.92887308\\
    -0.24928267
    \end{bmatrix}.
\end{equation}
The entries correspond, respectively, to the old first mode, the old second mode, and the new residual direction. Thus the weak retained slot is still dominated by the old second mode, with magnitude $0.929$, while only $0.249$ of the slot points along the residual direction.

For the SPIN model, $\Phi_{\mathrm{pre}}=\Phi_{\mathrm{in,pre}}$ and $\Sigma_{\mathrm{pre}}=\Sigma_{\mathrm{in,pre}}$, giving
\begin{equation}
    K_{\mathrm{in}}
    =
    \begin{bmatrix}
    2.47602980 & 0 & -2.52698551\\
    0 & 0.03093119 & -0.30864925\\
    0 & 0 & 0.41701624
    \end{bmatrix}.
\end{equation}
The corresponding retained core vector is
\begin{equation}
    u^{(2)}_{K,\mathrm{in}}
    =
    \begin{bmatrix}
    0.10526228\\
    -0.59441607\\
    0.79723860
    \end{bmatrix}.
\end{equation}
Now the same weak slot is much more strongly aligned with the residual direction: the residual contribution has magnitude $0.797$, while the old second-mode contribution has dropped to $0.594$. The signs are arbitrary because singular vectors are sign-indeterminate; the magnitudes show how the fixed rank-two budget is allocated. In this sense, in-span preconditioning does not add a third basis vector. It repurposes the weak retained slot so that the same out-of-span correction survives the truncation more effectively.

\section{Alternative fixed-rank updates in the spiral example}
\label{app:alternative_updates}

The spiral also allows comparison with several alternative fixed-rank update rules. All methods receive the same correction $q_{\mathrm{corr}}$ and use the same rank budget.

\subsection{Rank-one interpolation}

The rank-one interpolation update of Ref.~\cite{huang2023predictive} takes the form
\begin{equation}
    \Phi^+_{\mathrm{raw}}
    =
    \Phi_0
    +
    \frac{r\alpha^\top}{\alpha^\top\alpha},
    \qquad
    \alpha=\Phi_0^\top q_{\mathrm{corr}}.
\end{equation}
This update is constructed so that
\begin{equation}
    \Phi^+_{\mathrm{raw}}\alpha
    =
    \Phi_0\alpha+r
    =
    q_{\mathrm{corr}} .
\end{equation}
Therefore the correction lies exactly in the updated span. However, exact interpolation of one state does not imply that the updated basis is well aligned with the new residual direction.

\subsection{Greedy residual replacement}

A greedy residual replacement keeps the dominant old mode and replaces the weak mode by the normalized residual:
\begin{equation}
    \phi^+_{1,\mathrm{greedy}}=\phi^0_1,
    \qquad
    \phi^+_{2,\mathrm{greedy}}=\widehat{r}.
\end{equation}
This achieves perfect residual alignment, but discards the in-plane contribution associated with the second old mode. In the spiral, the correction error is therefore
\begin{equation}
    \|q_{\mathrm{corr}}-q^6_{\mathrm{ROM,greedy}}\|_2
    =
    |\alpha_2|
    =
    1.49068781 .
\end{equation}

\subsection{Optimal one-slot rotation}

If the dominant first mode is fixed and only the weak second slot is allowed to rotate in the plane spanned by $\phi_2^0$ and $\widehat{r}$, the exact-correction update is
\begin{equation}
    \phi^+_{2,\mathrm{opt}}
    =
    \frac{\alpha_2\phi_2^0+r}
    {\sqrt{\alpha_2^2+\rho^2}} .
\end{equation}
This update makes the correction exact while keeping $\phi_1^0$ fixed. Its residual alignment is
\begin{equation}
    \frac{\rho}{\sqrt{\alpha_2^2+\rho^2}}
    =
    0.26940447 .
\end{equation}

The comparison clarifies the role of in-span learning. It does not optimize correction interpolation alone, nor residual alignment alone. Instead, it brings a trajectory-preconditioned basis into the same truncated update, which allows the weak modal slot to be repurposed without discarding essential in-plane content. To compare the alternatives, we use the same residual capture, plane angle, retained alignment, and correction-error diagnostics defined above. We also report the slot angle
\[
    \theta_{\mathrm{slot}}
    =
    \arccos\left(
    \left|
    (\phi^{\mathrm{pre}}_2)^\top \phi^{\mathrm{post}}_2
    \right|
    \right),
\]
where $\phi^{\mathrm{pre}}_2$ is the pre-update weak mode and $\phi^{\mathrm{post}}_2$ is the post-update vector assigned to that weak modal slot. The absolute value removes the arbitrary sign of basis vectors. A small slot angle means that the weak slot remains close to its old direction, whereas a large slot angle indicates that the slot has been substantially repurposed. The resulting comparison is given in \cref{tab:supp_alternative_updates} and visualized in \cref{fig:supp_pareto}.

\begin{table}[!ht]
\centering
\caption{\textbf{Comparison of fixed-rank updates on the spiral correction.}}
\label{tab:supp_alternative_updates}
\small
\setlength{\tabcolsep}{4.5pt}
\begin{tabular}{lccccc}
\toprule
Method & Capture & Plane angle & Slot angle & Retained alignment & Correction error \\
\midrule
Baseline adaptive & 0.260 & $15.06^\circ$ & $21.74^\circ$ & 0.249 & $1.02\times 10^{-2}$ \\
SPIN & 0.802 & $53.29^\circ$ & $53.53^\circ$ & 0.797 & $8.84\times 10^{-4}$ \\
Rank-one interpolation & 0.162 & $9.30^\circ$ & $5.43^\circ$ & 0.132 & 0 \\
Greedy replacement & 1.000 & $90.00^\circ$ & $90.00^\circ$ & 1.000 & 1.491 \\
Optimal slot rotation & 0.269 & $15.63^\circ$ & $15.63^\circ$ & 0.269 & 0 \\
\bottomrule
\end{tabular}
\end{table}

\begin{figure}[!t]
\centering
\includegraphics[width=0.78\linewidth]{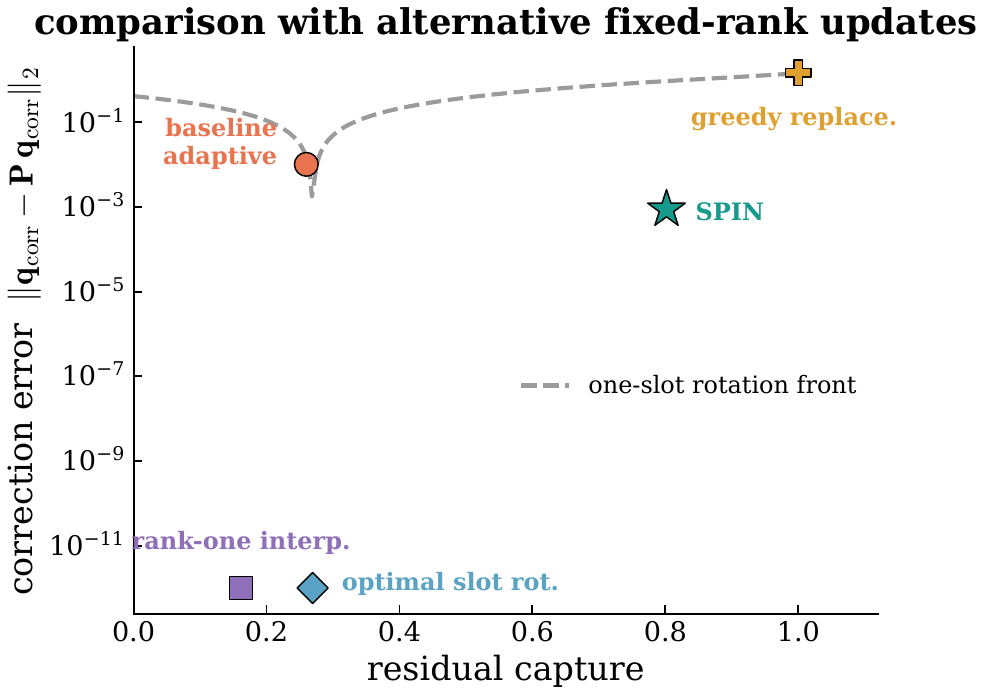}
\caption{\textbf{Alternative fixed-rank updates expose the tradeoff between fitting the correction and retaining the residual.}
Rank-one interpolation and optimal one-slot rotation can fit the correction exactly, but do not strongly align the basis with the new residual direction. Greedy residual replacement captures the residual perfectly, but discards important in-plane content. The SPIN update occupies a different part of the tradeoff because the basis has already been organized by the ROM trajectory before the correction arrives.}
\label{fig:supp_pareto}
\end{figure}

\section{Hyperparameter sensitivity and additional PDE results}
\label{app:hyperparameter_sensitivity}

The SPIN framework is governed by three main parameters: the in-span forgetting factor $\gamma_{\rm in}$, the out-of-span forgetting factor $\gamma_{\rm out}$, and the out-of-span correction interval $z_s$. For each benchmark, we tune the forgetting factors by minimizing the mean relative $L_2$ error over the online trajectory. The baseline adaptive is tuned over its own $\gamma_{\rm out}$ values, and the SPIN model is tuned over the two-dimensional $(\gamma_{\rm in},\gamma_{\rm out})$ grid (\cref{fig:supp_sensitivity,fig:supp_ffout_only}). The main-text comparisons therefore use the best-performing version of each adaptive strategy.

For Burgers, the best baseline adaptive model uses $\gamma_{\rm out}=0.01$, while the best SPIN model uses $(\gamma_{\rm in},\gamma_{\rm out})=(1.0,0.25)$. For Fisher--KPP, the best baseline adaptive model uses $\gamma_{\rm out}=0.1$, while the best SPIN model uses $(\gamma_{\rm in},\gamma_{\rm out})=(0.9,0.25)$. These values are reported in the main text and repeated in \cref{tab:supp_benchmark_parameters}.

The sensitivity maps also clarify the role of in-span forgetting. The best SPIN models occur at large $\gamma_{\rm in}$, meaning that the PDE benchmarks do not benefit from aggressive in-span forgetting. This is consistent with the spectral diagnostics in the main text: unlike the first spiral correction, the PDE trajectories carry enough energy in the retained modes that the singular values are reinforced rather than suppressed. The role of in-span learning in these cases is therefore not to erase a stale mode, but to continuously condition the basis spectrum and orientation using the ROM trajectory.

The maps show a useful practical pattern. When $\gamma_{\rm in}$ is large, the model can become sensitive to the out-of-span forgetting factor $\gamma_{\rm out}$, and poor choices of $\gamma_{\rm out}$ may lead to much larger errors. When $\gamma_{\rm in}$ is smaller, the performance is often more robust across $\gamma_{\rm out}$, although not always optimal. Thus, aggressive in-span memory can produce the best models when paired with a suitable out-of-span forgetting factor, while smaller $\gamma_{\rm in}$ may provide a safer but more conservative operating regime.

The correction interval $z_s$ controls how long the ROM evolves between external corrections. Larger $z_s$ reduces the number of full-order operator queries, but gives the ROM more time to drift; smaller $z_s$ provides more frequent out-of-span information but increases the hybrid cost. \Cref{fig:supp_sensitivity} shows that adding the in-span channel shifts this tradeoff. For a fixed target error, the SPIN ROM can use a substantially larger $z_s$ than the baseline adaptive ROM, meaning that it can evolve for longer on its own before requiring new external information. This supports the central interpretation of in-span learning as trajectory-informed basis preparation: between correction events, the ROM uses its own predictions to keep the basis better organized, so external corrections can be queried less frequently for the same accuracy level. The interval cannot be made arbitrarily large, however, because if the ROM trajectory drifts too far, its in-span predictions no longer provide a reliable guide to the local dynamics.

\begin{figure}[!t]
\centering
\begin{subcaptionblock}{0.32\linewidth}\centering
  \includegraphics[width=\linewidth]{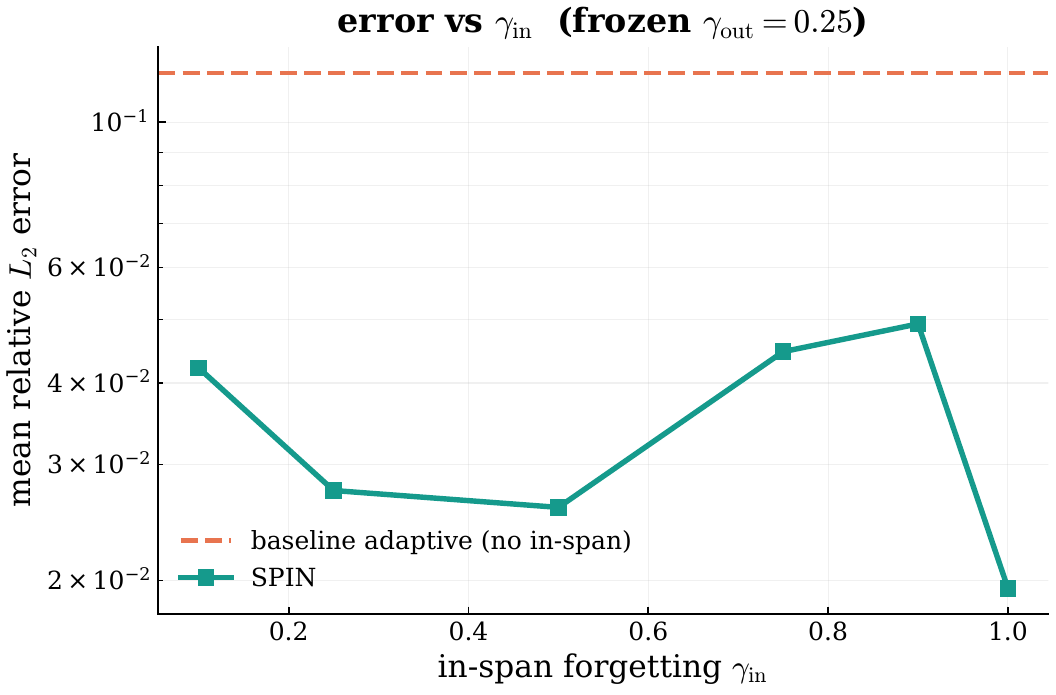}
  \caption{Burgers: $\gamma_{\rm in}$}\end{subcaptionblock}\hfill
\begin{subcaptionblock}{0.32\linewidth}\centering
  \includegraphics[width=\linewidth]{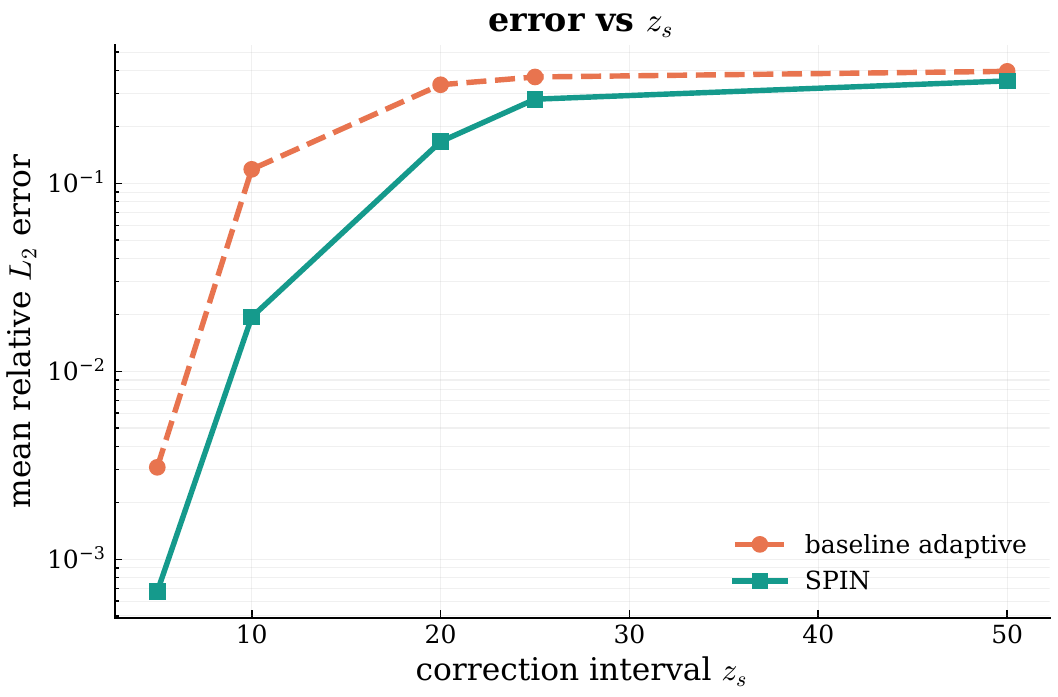}
  \caption{Burgers: $z_s$}\end{subcaptionblock}\hfill
\begin{subcaptionblock}{0.32\linewidth}\centering
  \includegraphics[width=\linewidth]{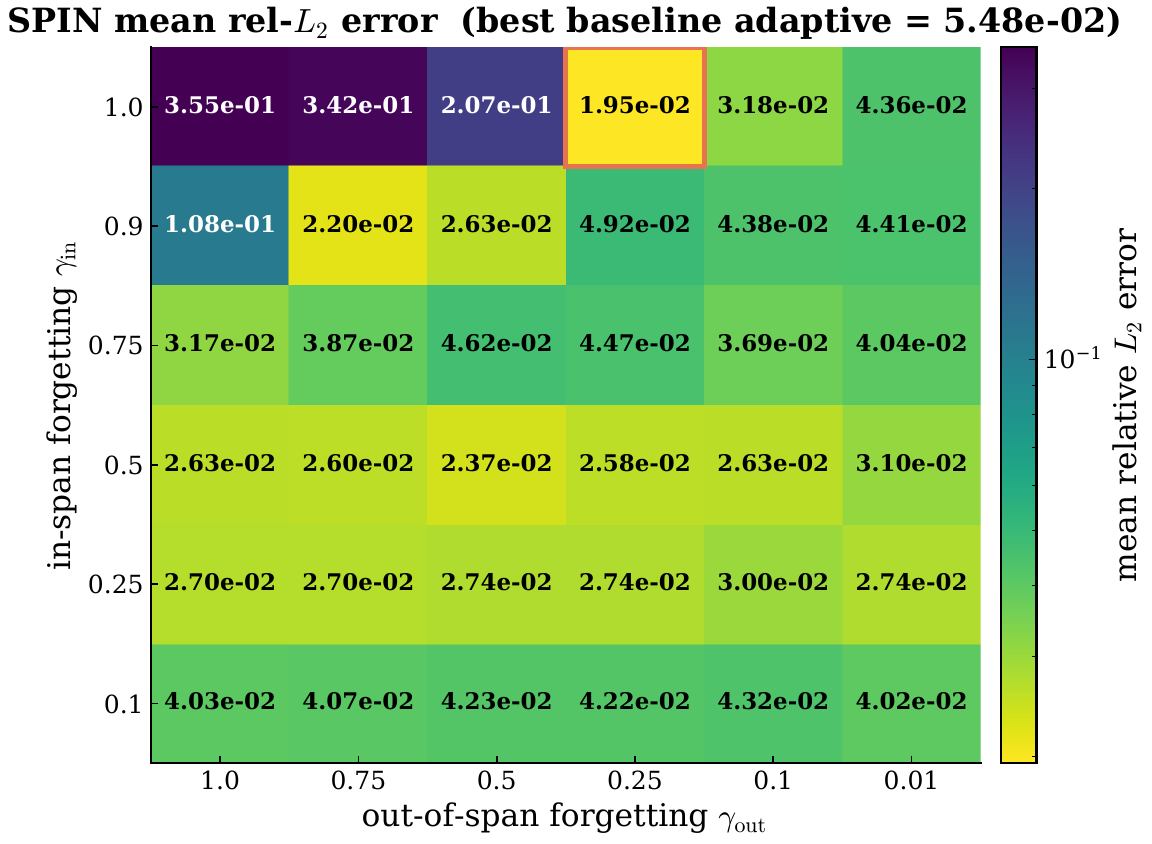}
  \caption{Burgers: $\gamma_{\rm in}\times\gamma_{\rm out}$}\end{subcaptionblock}

\vspace{0.6em}
\begin{subcaptionblock}{0.32\linewidth}\centering
  \includegraphics[width=\linewidth]{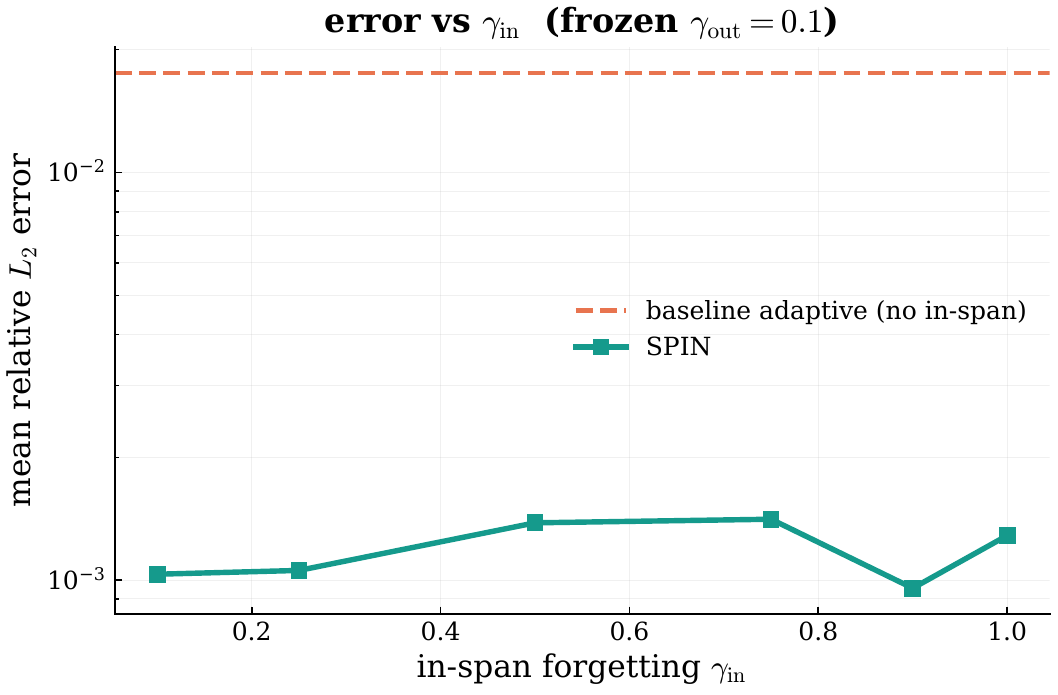}
  \caption{Fisher--KPP: $\gamma_{\rm in}$}\end{subcaptionblock}\hfill
\begin{subcaptionblock}{0.32\linewidth}\centering
  \includegraphics[width=\linewidth]{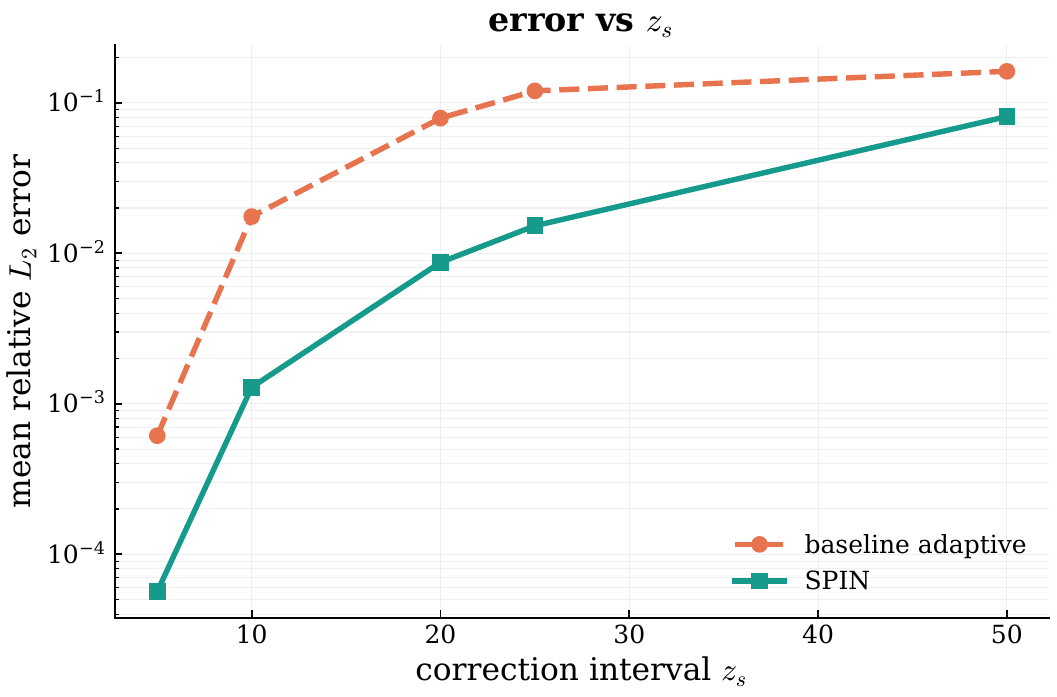}
  \caption{Fisher--KPP: $z_s$}\end{subcaptionblock}\hfill
\begin{subcaptionblock}{0.32\linewidth}\centering
  \includegraphics[width=\linewidth]{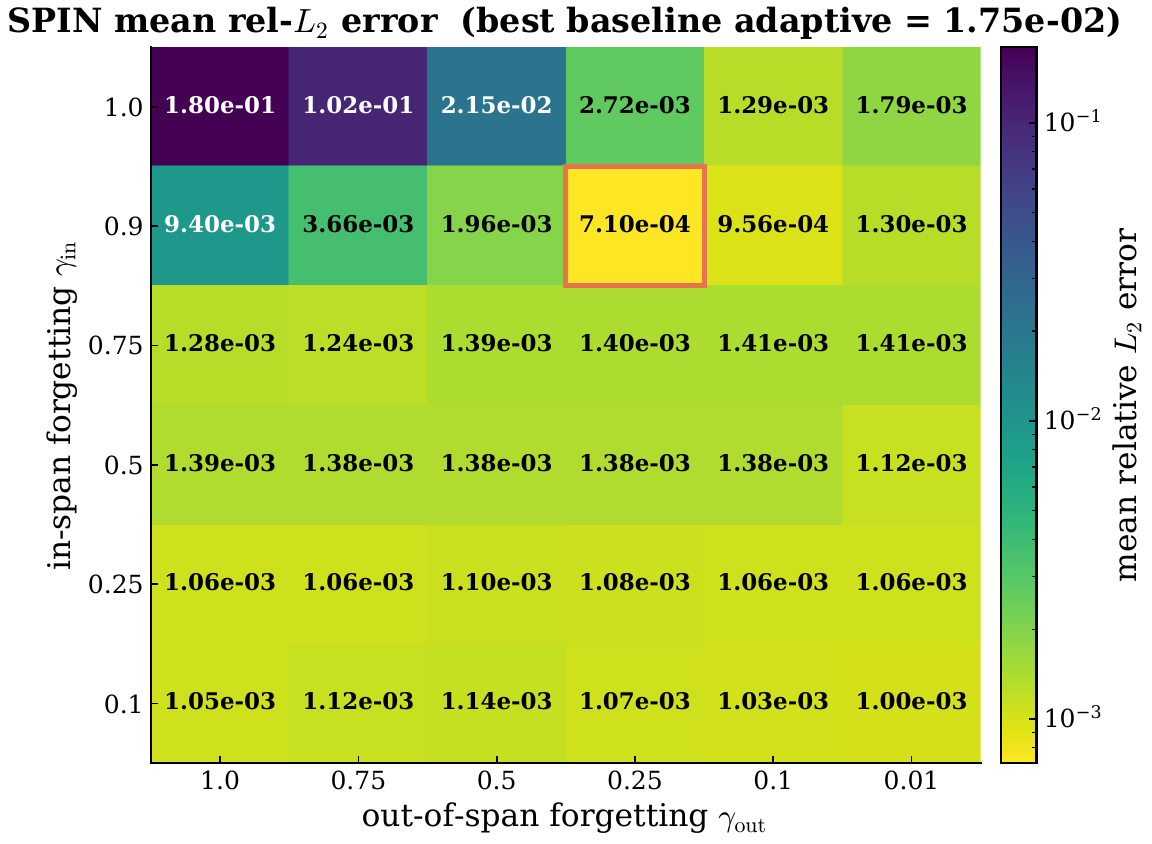}
  \caption{Fisher--KPP: $\gamma_{\rm in}\times\gamma_{\rm out}$}\end{subcaptionblock}
\caption{\textbf{Sensitivity to in-span forgetting and correction interval.}
When $\gamma_{\rm in}$ is close to one, in-span forgetting is weak and the ROM retains a longer memory of its own trajectory; when $\gamma_{\rm in}$ is small, the in-span update emphasizes recent states more aggressively. The most useful regime balances trajectory memory, spectral reweighting and stability between correction events. In the $\gamma_{\rm in}\times\gamma_{\rm out}$ maps (panels c and f), the reference value reported in the title is the best baseline adaptive error obtained by sweeping its out-of-span forgetting factor (\cref{fig:supp_ffout_only}).}
\label{fig:supp_sensitivity}
\end{figure}

\begin{figure}[!t]
\centering
\begin{subcaptionblock}{0.49\linewidth}\centering
  \includegraphics[width=\linewidth]{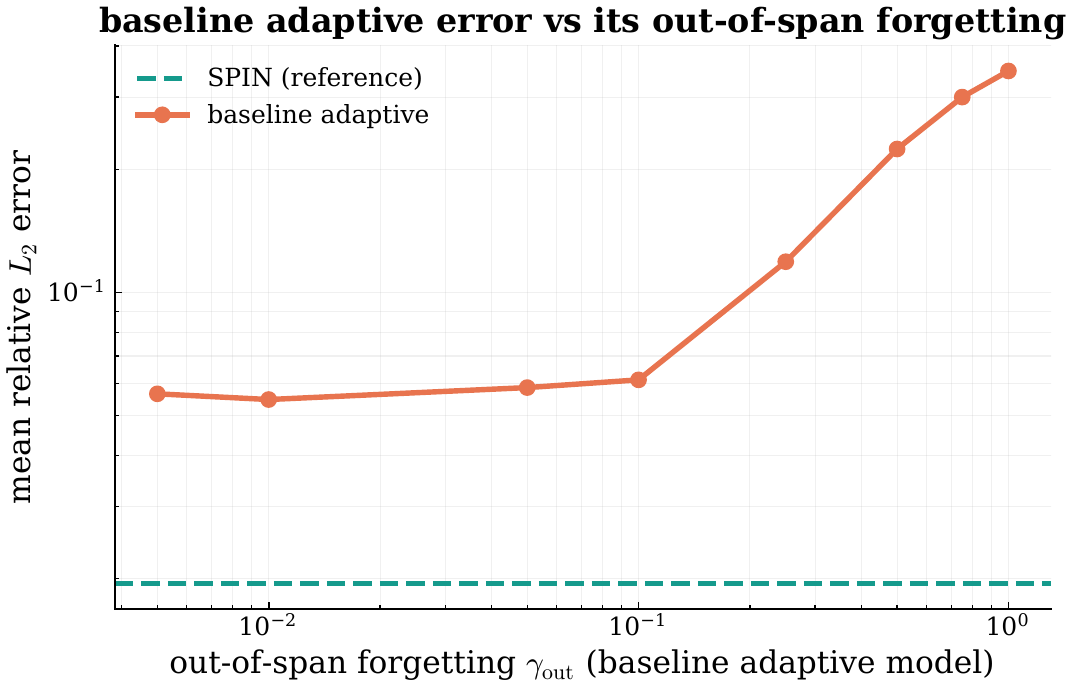}
  \caption{Burgers}\end{subcaptionblock}\hfill
\begin{subcaptionblock}{0.49\linewidth}\centering
  \includegraphics[width=\linewidth]{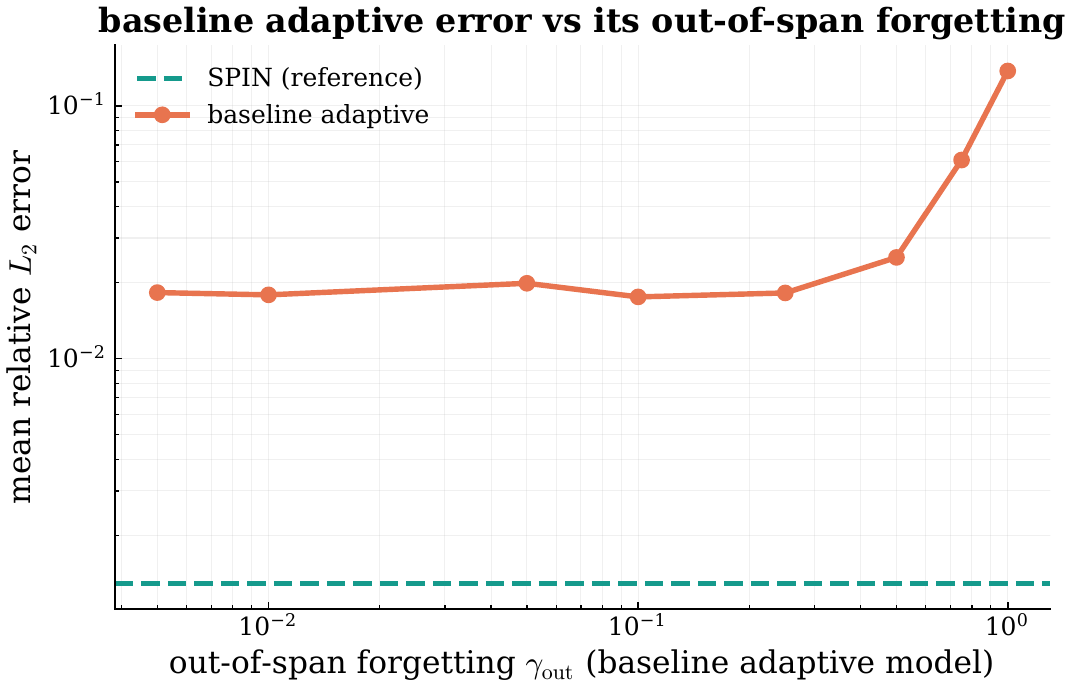}
  \caption{Fisher--KPP}\end{subcaptionblock}
\caption{\textbf{Out-of-span forgetting for the baseline adaptive model.}
The baseline adaptive model's mean relative $L_2$ error as a function of its
out-of-span forgetting factor $\gamma_{\rm out}$ (swept on a logarithmic axis),
with the SPIN error shown as a dashed reference line. The optimal
out-of-span forgetting is problem-dependent---near $\gamma_{\rm out}\approx 0.01$
for Burgers and $\gamma_{\rm out}\approx 0.1$ for Fisher--KPP---but in both cases
the best-tuned baseline adaptive ROM remains substantially less accurate than the
SPIN ROM.}
\label{fig:supp_ffout_only}
\end{figure}

\Cref{fig:supp_additional_pdes} reports additional Burgers and Fisher--KPP diagnostics, including robustness across initial conditions and final-time profiles for representative cases.

\begin{figure}[!t]
\centering
\begin{subcaptionblock}{0.49\linewidth}\centering
  \includegraphics[width=\linewidth]{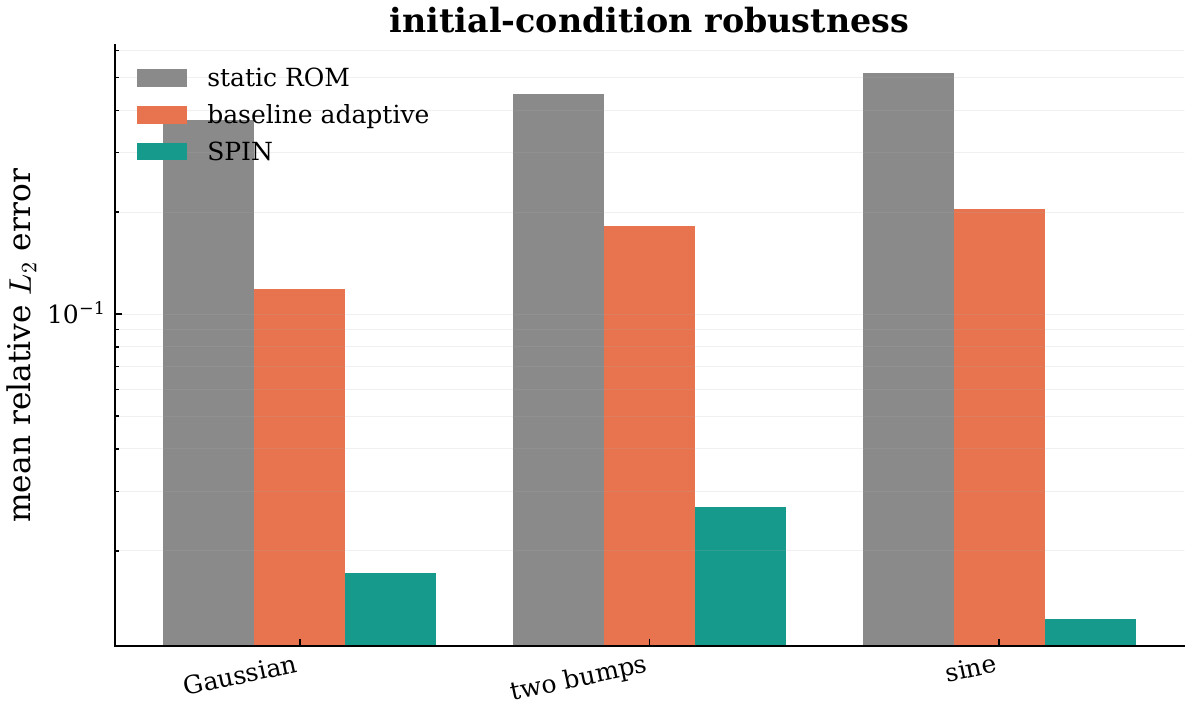}
  \caption{Burgers: initial-condition robustness}\end{subcaptionblock}\hfill
\begin{subcaptionblock}{0.49\linewidth}\centering
  \includegraphics[width=\linewidth]{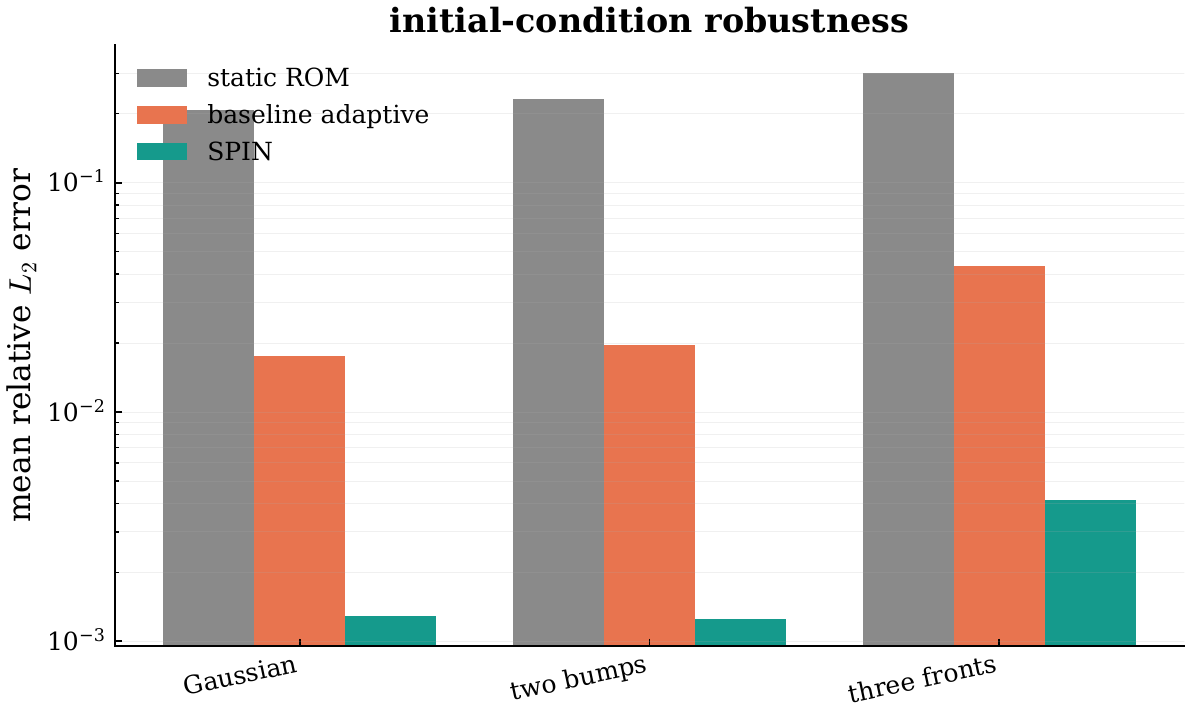}
  \caption{Fisher--KPP: initial-condition robustness}\end{subcaptionblock}

\vspace{0.6em}
\begin{subcaptionblock}{0.99\linewidth}\centering
  \includegraphics[width=\linewidth]{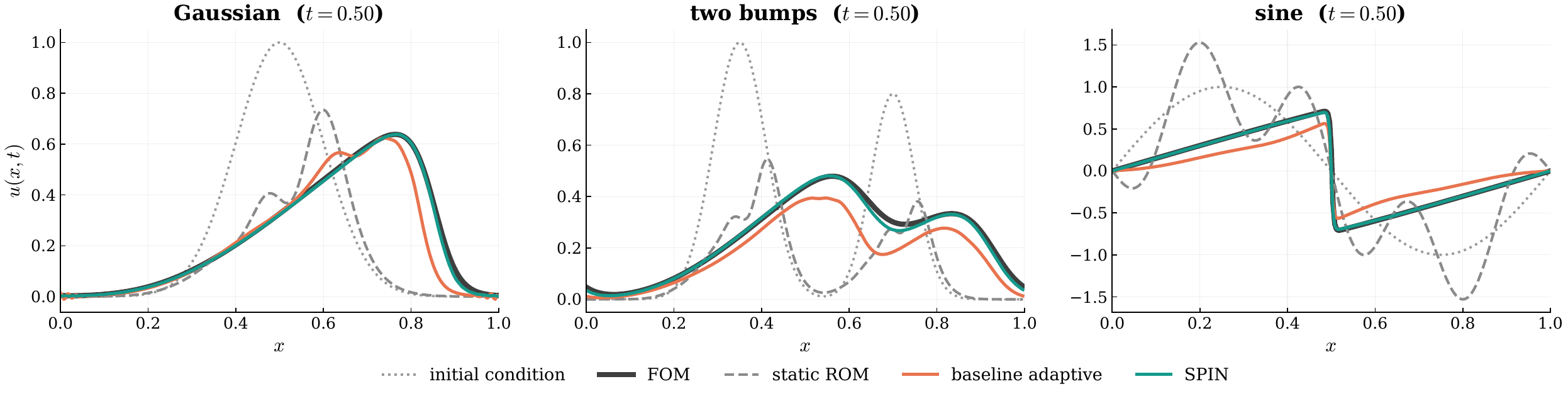}
  \caption{Burgers: final-time profiles across initial conditions}\end{subcaptionblock}

\vspace{0.4em}
\begin{subcaptionblock}{0.99\linewidth}\centering
  \includegraphics[width=\linewidth]{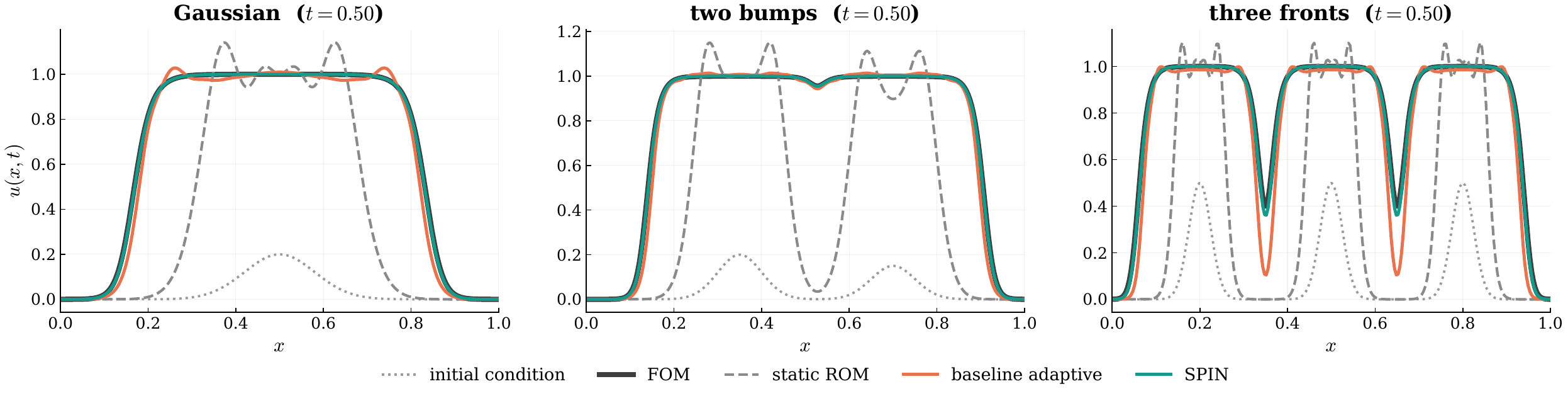}
  \caption{Fisher--KPP: final-time profiles across initial conditions}\end{subcaptionblock}
\caption{\textbf{Additional PDE results support the main-text comparisons.}
Initial-condition robustness tests for Burgers and Fisher--KPP. All ROMs are trained with rank $r=4$ and $m=4$ QDEIM sample points, with offline basis trained on the initial $4$ time steps from $t=0$ to $t=0.004$, while prediction is evaluated over $500$ time steps from $t=0$ to $t=0.50$, with external corrections every $z_s=10$ steps. Across these additional cases, the SPIN ROM uses the same out-of-span correction mechanism as the baseline adaptive ROM, but produces smaller errors and more accurate final-time profiles.}
\label{fig:supp_additional_pdes}
\end{figure}

\section{Reproducibility information}
\label{app:reproducibility}

The main numerical settings used to generate the spiral, Burgers, and Fisher--KPP results are collected in \cref{tab:supp_benchmark_parameters}. Together with the model definitions, update rules, and diagnostics described in the Methods and appendices above, as well as the source code available at \url{https://github.com/APHedayat/SPIN}, these parameters specify the benchmark configurations used for all results reported in the main text and appendices.

\begin{table}[!ht]
\centering
\caption{\textbf{Main benchmark parameters.}}
\label{tab:supp_benchmark_parameters}
\begin{tabular}{llll}
\toprule
Benchmark & Parameter & Value & Notes \\
\midrule
Spiral & $\alpha$ & $0.4$ & Spiral vertical drift \\
 & $\Delta t$ & $0.35$ & Exact discrete map \\
 & $t_0$ & $4.0$ & Initial time \\
 & $r$ & $2$ & Rank-2 ROM \\
 & $\gamma_{\rm in}$ & $0.1$ & In-span forgetting \\
\midrule
Burgers & $N_x$ & $1000$ & Periodic finite-difference grid \\
 & $\nu$ & $10^{-2}$ & Viscosity \\
 & $\Delta t$ & $10^{-3}$ & Backward Euler \\
 & $N_t$ & $500$ & Prediction steps \\
 & $r$ & $4$ & ROM rank \\
 & $m$ & $4$ & QDEIM samples \\
 & $z_s$ & $10$ & Correction interval \\
 & $\gamma_{\rm in},\gamma_{\rm out}$ & $1.0,0.25$ & Forgetting factors (SPIN) \\
 & $\gamma_{\rm out}$ & $0.01$ & Forgetting factor (baseline adaptive) \\
\midrule
Fisher--KPP & $N_x$ & $256$ & Periodic finite-difference grid \\
 & $D$ & $10^{-4}$ & Diffusion coefficient \\
 & $\beta$ & $20$ & Reaction rate \\
 & $\Delta t$ & $10^{-3}$ & Backward Euler \\
 & $N_t$ & $500$ & Prediction steps \\
 & $r$ & $4$ & ROM rank \\
 & $m$ & $4$ & QDEIM samples \\
 & $z_s$ & $25$ & Correction interval \\
 & $\gamma_{\rm in},\gamma_{\rm out}$ & $0.9,0.25$ & Forgetting factors (SPIN) \\
 & $\gamma_{\rm out}$ & $0.1$ & Forgetting factor (baseline adaptive) \\
\bottomrule
\end{tabular}
\end{table}

\clearpage
\section{Mathematical connection between in-span and in-context learning}
\label{app:icl}

In-context learning (ICL) broadly refers to the ability of a model to modify its inference-time behavior using information contained in the current context, without retraining the underlying offline model~\cite{brown2020language,xie2022explanation,akyurek2023learning}. Although most commonly associated with large language models, several analyses have shown that the mechanism can be understood more generally as an inference-time computation performed over a sequence of contextual representations. In particular, linearized attention admits a recurrent fast-weight interpretation in which context is accumulated through rank-one outer-product updates~\cite{katharopoulos2020linear,schlag2021fastweights}. We show here that the in-span update underlying SPIN has the same mathematical structure.

\subsection{Outer-product memory in linear attention}

Consider a sequence of representations $x^1,x^2,\ldots$. In a linear-attention layer, each representation generates a key, value, and query,
\begin{equation}
    k^j=\varphi(W_k x^j),\qquad
    v^j=W_v x^j,\qquad
    q^j=\varphi(W_q x^j),
    \label{eq:icl_kvq}
\end{equation}
where $\varphi$ denotes the feature map used to linearize attention. Ignoring normalization, the causal attention computation can be written in terms of a recurrent context matrix
\begin{equation}
    F_j=F_{j-1}+v^j (k^j)^\top ,
    \label{eq:icl_fastweight}
\end{equation}
with the corresponding output obtained by querying this accumulated state,
\begin{equation}
    \widetilde{y}^{\,j}=F_j q^j .
    \label{eq:icl_output}
\end{equation}
Thus, rather than explicitly retaining the entire preceding sequence, the model can carry forward a finite-dimensional state $F_j$ containing a sum of outer products generated from the preceding context.

For comparison with the forgetting mechanism used by SPIN, consider the exponentially weighted form
\begin{equation}
    F_j=\lambda F_{j-1}+v^j (k^j)^\top ,\qquad 0<\lambda\le 1,
    \label{eq:icl_fastweight_decay}
\end{equation}
which gives
\begin{equation}
    F_s=\lambda^s F_0+\sum_{j=1}^{s}\lambda^{\,s-j} v^j (k^j)^\top
    \label{eq:icl_fastweight_expanded}
\end{equation}
after applying $s$ representations starting from $F_0$. The choice $\lambda=1$ recovers the purely additive fast-weight accumulation in \cref{eq:icl_fastweight}; introducing $\lambda<1$ simply gives greater weight to recent context.

\subsection{The in-span state is a symmetric fast-weight memory}

Between external corrections, SPIN receives only states generated by the ROM itself. Let $a^j\in\mathbb{R}^r$ denote the reduced coordinates of the $j$th in-span prediction. As derived in the main text, the reduced covariance evolves according to
\begin{equation}
    C_j=\gamma_{\rm in}^2 C_{j-1}+a^j (a^j)^\top ,
    \label{eq:icl_inspan_recursion}
\end{equation}
and therefore
\begin{equation}
    C_s=\gamma_{\rm in}^{2s} C_0+\sum_{j=1}^{s}\gamma_{\rm in}^{2(s-j)} a^j (a^j)^\top .
    \label{eq:icl_inspan_expanded}
\end{equation}
Comparing \cref{eq:icl_fastweight_decay} and \cref{eq:icl_inspan_recursion} shows an exact correspondence under
\begin{equation}
    \lambda=\gamma_{\rm in}^2,\qquad
    v^j=k^j=a^j,\qquad
    F_j=C_j .
    \label{eq:icl_correspondence}
\end{equation}
Thus, the SPIN covariance is a symmetric fast-weight memory whose context consists of the model's own reduced trajectory. The initial state $C_0=\Sigma_0^2$ plays the role of a prior spectral memory inherited from the offline model, while each new prediction contributes the rank-one contextual update $a^j (a^j)^\top$. Forgetting determines how rapidly earlier context is discounted.

In this sense, in-span learning can be viewed as a form of \emph{self-generated in-context learning for dynamical systems}, where the surrogate uses the context created by its own evolving trajectory to reorganize its internal representation and thereby change how it responds to future information.

\end{document}